\documentclass[11pt]{article}
\usepackage{amsmath,amssymb,amsthm}
\usepackage{latexsym}
\usepackage[dvipdfmx]{graphicx}
\usepackage{algorithm}
\usepackage{algorithmic} 
\usepackage{authblk}
\usepackage{color}
\usepackage{relsize}
\usepackage{url}
\usepackage[colorlinks,linkcolor=blue,anchorcolor=green,citecolor=green]{hyperref}
\usepackage{caption} 
\usepackage{tikz}
\usepackage{yhmath}
\renewcommand{\arraystretch}{1.3}
 
\newtheorem{thm}{Theorem}

\newtheorem{definition}{Definition}

\newtheorem{example}{Example}
\newtheorem{assumption}{Assumption}
\usepackage{color}

\newcommand{\neff}{n_{\mathsmaller{\textit{eff}}}}

\newcommand{\eins}{\boldsymbol{1}}
\newcommand{\argmin}{\operatornamewithlimits{arg \, min}}
\DeclareMathOperator{\sign}{sign}

\newcommand\firstcircle{(6,-4) circle (1.5cm)}
\newcommand\firstellipse{(4.5,-4) ellipse (5cm and 2cm)}
\newcommand\secondellipse{(7.5,-4cm) ellipse (5cm and 2cm)}
\newcommand\thirdellipse{(6,-4cm) ellipse (3.6cm and 1.8cm)}
\newcommand\fourthellipse{(6,-4cm) ellipse (6cm and 2.1cm)}

\title{Learning theory estimates with observations from general stationary stochastic processes}

\author[1]{Hanyuan Hang}
\author[1]{Yunlong Feng}
\author[2]{Ingo Steinwart}
\author[1]{Johan A.K. Suykens}

\affil[1]{Department of Electrical Engineering, KU Leuven,  Leuven, Belgium}
\affil[2]{Institute for Stochastics and Applications, University of Stuttgart, Stuttgart, Germany}

\date{}

\begin{document}
\maketitle

\begin{abstract} 
This paper investigates the supervised learning problem with observations drawn from certain general stationary stochastic processes. Here by \emph{general}, we mean that many stationary stochastic processes can be included. We show that when the stochastic processes satisfy a generalized Bernstein-type inequality, a unified treatment on analyzing the learning schemes with various mixing processes can be conducted and a sharp oracle inequality for generic regularized empirical risk minimization schemes can be established. The obtained oracle inequality is then applied to derive convergence rates for several learning schemes such as empirical risk minimization (ERM), least squares support vector machines (LS-SVMs) using given generic kernels, and SVMs using Gaussian kernels for both least squares and quantile regression. It turns out that for i.i.d.~processes, our learning rates for ERM recover the optimal rates. On the other hand, for non-i.i.d.~processes including geometrically $\alpha$-mixing Markov processes, geometrically $\alpha$-mixing processes with restricted decay, $\phi$-mixing processes, and (time-reversed) geometrically $\mathcal{C}$-mixing processes, our learning rates for SVMs with Gaussian kernels match, up to some arbitrarily small extra term in the exponent, the optimal rates. For the remaining cases, our rates are at least close to the optimal rates. As a by-product, the assumed generalized Bernstein-type inequality also provides an interpretation of the so-called ``effective number of observations" for various mixing processes.  
\end{abstract}

\section{Introduction and Motivation	} 

In this paper, we study the supervised learning problem which aims at inferring a functional relation between explanatory variables and response variables \cite{Vapnik98a}. In the literature of statistical learning theory, one of the main research topics is the generalization ability of different learning schemes which indicate their learnabilities on future observations. Nowadays, it has been well understood that the Bernstein-type inequalities play an important role in deriving fast learning rates. For example, the analysis of various algorithms from non-parametric statistics and machine learning crucially depends on these inequalities, see  e.g.~\cite{DeGyLu96,DeLu01,GyKoKrWa02,StCh08a}. Here, stronger results can typically be achieved since the Bernstein-type inequality allows for localization due to its specific dependence on the variance. In particular, most derivations of minimax optimal learning rates are based on it.

The classical Bernstein inequality assumes that the data are generated by an i.i.d.~process. Unfortunately, however, this assumption is often violated in many real-world applications including financial prediction, signal processing, system identification and diagnosis, text and speech recognition, and time series forecasting, among others. For this and other reasons, there has been some effort to establish Bernstein-type inequalities for non-i.i.d.~processes. For instance, generalizations of Bernstein-type inequalities to the cases of $\alpha$-mixing \cite{Rosenblatt56a} and $\mathcal{C}$-mixing \cite{HaSt15a} processes have been found  \cite{Bosq93a,MoMa96a,MePeRi09a,Samson00a} and \cite{HaSt15a, hanglearning}, respectively. These Bernstein-type inequalities have been applied to derive various convergence rates. For example, the Bernstein-type inequality established in \cite{Bosq93a} was employed in \cite{Zhang04a} to derive convergence rates for sieve estimates from strictly stationary $\alpha$-mixing processes in the special case of neural networks.  \cite{HaSt14a} applied the Bernstein-type inequality in \cite{MoMa96a} to derive an oracle inequality (see Page $220$ in \cite{StCh08a} for the meaning of the \textit{oracle inequality}) for generic regularized empirical risk minimization algorithms with stationary $\alpha$-mixing processes. By applying the Bernstein-type inequality in \cite{MePeRi09a}, \cite{Belomestny11a} derived almost sure uniform  convergence rates for the estimated L\'{e}vy density both in mixed-frequency and low-frequency setups and proved their optimality in the minimax sense. Particularly, concerning the least squares loss, \cite{AlWi12a} obtained the optimal learning rates for $\phi$-mixing processes by applying the Bernstein-type inequality established in \cite{Samson00a}. By developing a Bernstein-type inequality for $\mathcal{C}$-mixing processes that include $\phi$-mixing processes and many discrete-time dynamical systems, \cite{HaSt15a} established an oracle inequality as well as fast learning rates for generic regularized empirical risk minimization algorithms with observations from $\mathcal{C}$-mixing processes.  

The above-mentioned inequalities are termed as Bernstein-type since they rely on the variance of the random variables. However, we note that these inequalities are usually presented in similar but rather complicated forms which consequently are not easy to apply directly in analyzing the performance of statistical learning schemes and may be also lack of interpretability. On the other hand, existing studies on learning from mixing processes may diverse from one to another since they may be conducted under different assumptions and notations, which leads to barriers in comparing the learnability of these learning algorithms.  

In this work, we first introduce a generalized Bernstein-type inequality and show that it can be instantiated to various stationary mixing processes. Based on the generalized Bernstein-type inequality, we establish an oracle inequality for a class of learning algorithms including ERM \cite[Chapter 6]{StCh08a} and SVMs. On the technical side, the oracle inequality is derived by refining and extending the analysis of \cite{StCh09a}. To be more precise, the analysis in \cite{StCh09a} partially ignored localization with respect to the regularization term, which in our study is addressed by a carefully arranged peeling approach inspired by \cite{StCh08a}. This leads to a sharper stochastic error bound and consequently a sharper bound for the oracle inequality, comparing with that of \cite{StCh09a}. Besides, based on the assumed generalized Bernstein-type inequality, we also provide an interpretation and comparison of the effective numbers of observations when learning from various mixing processes. 

Our second main contribution made in the present study lies in that we present a unified treatment on analyzing learning schemes with various mixing processes. For example, we establish fast learning rates for $\alpha$-mixing and (time-reversed) $\mathcal{C}$-mixing processes by tailoring the generalized oracle inequality. For ERM, our results  match those in the i.i.d.~case, if one replaces the number of observations with the effective number of observations. For LS-SVMs, as far as we know, the best learning rates for the case of geometrically $\alpha$-mixing process are those derived in \cite{XuCh08a, SuWu10a, Feng12a}. When applied to LS-SVMs, it turns out that our oracle inequality leads to faster learning rates that those reported in \cite{XuCh08a} and \cite{Feng12a}. For sufficiently smooth kernels, our rates are also faster than those in \cite{SuWu10a}. For other mixing processes including geometrically $\alpha$-mixing Markov chains, geometrically $\phi$-mixing processes, and geometrically $\mathcal{C}$-mixing processes, our rates for LS-SVMs with Gaussian kernels match essentially the optimal learning rates, while for LS-SVMs with given generic kernel, we only obtain rates that are close to the optimal rates.

The rest of this work is organized as follows: In Section \ref{definitions}, we introduce some basics of statistical learning theory.  Section \ref{GeneralizedBernstein} presents the key assumption of a generalized Bernstein-type inequality for stationary mixing processes, and present some concrete examples that satisfy this assumption. Based on the generalized Bernstein-type inequality, a sharp oracle inequality is developed in Section \ref{MainResults} while its proof is deferred to the Appendix. Section \ref{applications} provides some applications of the newly developed oracle inequality. The paper is ended in Section \ref{Conclusion}.

\section{A Primer in Learning Theory}\label{definitions}

Let $(X, \mathcal{X})$ be a measurable space and $Y \subset \mathbb{R}$ be a closed subset. The goal of (supervised) statistical learning is to find a function $f : X \to \mathbb{R}$ such that for $(x, y) \in X \times Y$ the value $f(x)$ is a good prediction of $y$ at $x$. 
The following definition will help us define what we mean by ``good''. 

\begin{definition}
Let $(X, \mathcal{X})$ be a measurable space and $Y \subset \mathbb{R}$ be a closed subset. 
Then a function $L : X \times Y \times \mathbb{R} \to [0, \infty)$ 
is called a loss function, or simply a loss, if it is measurable.
\end{definition}

In this study, we are interested in loss functions that in some sense can be restricted to domains of the form $X \times Y \times [-M, M]$ as defined below, which is typical in learning theory \cite[Definition 2.22]{StCh08a} and is in fact motivated by the boundedness of $Y$.

\begin{definition}
We say that a loss $L : X \times Y \times \mathbb{R} \to [0, \infty)$ can be clipped at $M > 0$, if, for all $(x, y, t) \in X \times Y \times \mathbb{R}$, we have 
\begin{align*}
L(x, y, \wideparen{t} \,) \leq L(x, y, t), 
\end{align*}
where $\wideparen{t}$ denotes the clipped value of $t$ at $\pm M$, that is
\begin{align*}
\wideparen{t} := 
\begin{cases}
- M & \text{ if } t < - M, \\
t   & \text{ if } t \in [-M, M], \\
M   & \text{ if } t > M.
\end{cases}
\end{align*}
\end{definition}

Throughout this work, we make the following assumptions on the loss function $L$:
\begin{assumption}\label{assumptionL}
 The loss function $L : X \times Y \times \mathbb{R} \to [0, \infty)$ can be clipped at some $M > 0$. Moreover, it is both bounded 
in the sense of $L(x, y, t) \leq 1$ and locally Lipschitz continuous, that is, 
\begin{align}\label{lipschitz}
|L(x, y, t) - L(x, y, t')| \leq |t - t'|\, . 
\end{align}
Here both inequalites are supposed to hold for all $(x, y) \in X \times Y$ and $t, t' \in [- M, M]$.
\end{assumption}

Note that the above assumption with Lipschitz constant equals to one can typically be enforced by scaling. To illustrate the generality of the above assumptions on $L$, let us first consider the case of binary classification, that is $Y := \{ - 1, 1 \}$. For this learning problem one often uses a convex surrogate for the original discontinuous classification loss $\eins_{(-\infty,0]}(y\sign(t))$, since the latter may lead to computationally infeasible approaches. Typical surrogates $L$ belong to the class of margin-based losses, that is,
$L$ is of the form $L(y, t) = \varphi(y t)$, where $\varphi : \mathbb{R} \to [0, \infty)$ is a suitable, convex function. Then $L$ can be clipped, if and only if $\varphi$ has a global minimum, see \cite[Lemma 2.23]{StCh08a}. In particular, the hinge loss, the least squares loss for classification, and the squared hinge loss can be clipped, but the logistic loss for classification and the AdaBoost loss cannot be clipped. On the other hand, \cite{Steinwart09a} established a simple technique, which is similar to inserting a small amount of noise into the labeling process, to construct a clippable modification of an arbitrary convex, margin-based loss. Finally, both the Lipschitz continuity and the boundedness of $L$ can be easily verified for these losses, where for the latter it may be necessary to suitably scale the loss.

Bounded regression is another class of learning problems, where the assumptions made on $L$ are often satisfied. Indeed, if $Y := [- M, M]$ and $L$ is a convex, distance-based loss represented by some $\psi: \mathbb{R} \to [0, \infty)$, that is $L(y, t) = \psi(y - t)$, 
then $L$ can be clipped whenever $\psi(0) = 0$, see again \cite[Lemma 2.23]{StCh08a}. In particular, the least squares loss 
\begin{align}\label{lsloss} 
L(y, t) = (y - t)^2     
\end{align}
and the $\tau$-pinball loss 
\begin{align}\label{pbloss} 
L_\tau(y,t) := \psi(y-t) = 
\begin{cases}
- (1 - \tau) (y - t), & \text{ if } y - t < 0, \\ 
\tau (y - t),         & \text{ if } y - t \geq 0, 
\end{cases}
\end{align}
used for quantile regression can be clipped. Again, for both losses, the Lipschitz continuity and the boundedness can be easily enforced by a suitable scaling of the loss.

Given a loss function $L$ and an $f : X \to \mathbb{R}$, we often use the notation $L \circ f$ for the function $(x, y) \mapsto L(x, y, f(x))$. Our major goal is to have a small average loss for future unseen observations $(x, y)$. This leads to the following definition.

\begin{definition}
Let $L : X \times Y \times \mathbb{R} \to [0, \infty)$ be a loss function and $P$ be a probability measure on $X \times Y$. Then, for a measurable function $f : X \to \mathbb{R}$, the $L$-risk is defined by 
\begin{align*}
\mathcal{R}_{L,P}(f) := \int\limits_{X \times Y} L(x, y, f(x)) \, d P(x, y).
\end{align*}
Moreover, the minimal $L$-risk
\begin{align*}
\mathcal{R}_{L,P}^* := \inf \{ \mathcal{R}_{L,P} (f)\, | \, f : X \to \mathbb{R} \text{ measurable} \}
\end{align*}
is called the Bayes risk with respect to $P$ and $L$. In addition, a measurable function $f_{L,P}^* : X \to \mathbb{R}$ satisfying $\mathcal{R}_{L,P}(f_{L,P}^*) = \mathcal{R}_{L, P}^*$ is called a Bayes decision function. 
\end{definition}

Let $(\Omega, \mathcal{A}, \mu)$ be a probability space, $\mathcal{Z} := (Z_i)_{i \geq 1}$ be an $X \times Y$-valued stochastic process on $(\Omega, \mathcal{A}, \mu)$, we write
\begin{displaymath}
   D := \bigl( (X_1, Y_1), \ldots, (X_n, Y_n) \bigr) := (Z_1, \ldots, Z_n) \in (X \times Y)^n
\end{displaymath}
for a training set of length $n$ that is distributed according to the first $n$ components of $\mathcal{Z}$.
Informally, the goal of learning from a training set $D$ is to find a decision function $f_D$ such that $\mathcal{R}_{L,P}(f_D)$ is close to the minimal risk $\mathcal{R}_{L, P}^*$. Our next goal is to formalize this idea. We begin with the following definition.

\begin{definition}
Let $X$ be a set and $Y \subset \mathbb{R}$ be a closed subset. A learning method $\mathcal{L}$ on $X \times Y$ maps every set $D \in (X \times Y)^n$, $n \geq 1$, to a function $f_D : X \to \mathbb{R}$.
\end{definition}

Now a natural question is whether the functions $f_D$ produced by a specific learning method satisfy 
\begin{align*}
\mathcal{R}_{L,P}(f_D) \to \mathcal{R}_{L, P}^*,\,\,\,\,\,\,\,\,n \to \infty \,.
\end{align*}
If this convergence takes place for all $P$, then the learning method is called universally consistent. In the i.i.d.~case many learning methods are known to be universally consistent, see e.g.~\cite{DeGyLu96} for classification methods, \cite{GyKoKrWa02} for regression methods, and \cite{StCh08a} for generic SVMs. For consistent methods, it is natural to ask how fast the convergence rate is. Unfortunately, 
in most situations uniform convergence rates are impossible, see \cite[Theorem 7.2]{DeGyLu96}, and hence establishing learning rates require some assumptions on the underlying distribution $P$. Again, results in this direction can be found in the above-mentioned books. In the non-i.i.d.~case, \cite{Nobel99a} showed that no uniform consistency is possible if one only assumes that the data generating process $\mathcal{Z}$ is stationary and ergodic. On the other hand, if some further assumptions of the dependence structure of $\mathcal{Z}$ are made, then consistency is possible, see e.g.~\cite{StHuSc09a}.  

Let us now describe the learning algorithms of particular interest to us. To this end, we assume that we have a hypothesis set $\mathcal{F}$ consisting of bounded measurable functions $f : X \to \mathbb{R}$, which is pre-compact with respect to the supremum norm $\|\cdot\|_{\infty}$. Since the cardinality of $\mathcal{F}$ can be infinite, we need to recall the following concept, which will enable us to approximate $\mathcal{F}$ by using finite subsets. 

\begin{definition}
Let $(T, d)$ be a metric space and $\varepsilon > 0$. We call $S \subset T$ an $\varepsilon$-net of $T$ if for all $t \in T$ there exists an 
$s \in S$ with $d(s, t) \leq \varepsilon$. Moreover, the $\varepsilon$-covering number of $T$ is defined by
\begin{align*}
\mathcal{N}(T,d,\varepsilon) := \inf \left\{ n \geq 1 : \exists s_1, \ldots, s_n \in T \text{ such that } T \subset \bigcup_{i=1}^n B_d(s_i, \varepsilon) \right\},
\end{align*}
where $\inf \emptyset := \infty$ and $B_d(s, \varepsilon) := \{ t \in T : d(t, s) \leq \varepsilon \}$ denotes the closed ball with center 
$s \in T$ and radius $\varepsilon$.
\end{definition}

Note that our hypothesis set $\mathcal{F}$ is assumed to be pre-compact, and hence for all $\varepsilon > 0$, the covering number $\mathcal{N}(\mathcal{F},\|\cdot\|_{\infty},\varepsilon)$ is finite. 

Denote $D_n := \frac{1}{n} \sum_{i=1}^n \delta_{(X_i, Y_i)}$, where $\delta_{(X_i, Y_i)}$ denotes the (random) Dirac measure at $(X_i, Y_i)$. In other words, $D_n$ is the empirical measure associated to the data set $D$. Then, the  risk of a function $f : X \to \mathbb{R}$ with respect to this measure
\begin{align*}
\mathcal{R}_{L,D_n}(f) = \frac{1}{n} \sum_{i=1}^n L(X_i,Y_i,f(X_i))
\end{align*}
is called the empirical $L$-risk.  

With these preparations we can now introduce the class of learning methods of interest:
\begin{definition}\label{crerm}
Let $L : X \times Y \times \mathbb{R} \to [0, \infty)$ be a loss that can be clipped at some $M > 0$, $\mathcal{F}$ be a hypothesis set, that is, a set of measurable functions $f : X \to \mathbb{R}$, with $0 \in \mathcal{F}$, and $\Upsilon$ be a regularizer on $\mathcal{F}$, 
that is, $\Upsilon : \mathcal{F} \to [0, \infty)$ with $\Upsilon(0) = 0$. Then, for $\delta \geq 0$, a learning method whose decision functions $f_{D_n, \Upsilon} \in \mathcal{F}$ satisfy
\begin{align}\label{deltaCRERM}
\Upsilon (f_{D_n,\Upsilon}) + \mathcal{R}_{L,D_n} (\wideparen{f}_{D_n,\Upsilon})\leq \inf_{f \in \mathcal{F}} \left( \Upsilon(f) + \mathcal{R}_{L,D_n}(f) \right) + \delta 
\end{align}
for all $n \geq 1$ and $D_n \in (X \times Y)^n$ is called $\delta$-approximate clipped regularized empirical risk minimization ($\delta$-CR-ERM) with respect to $L$, $\mathcal{F}$, and $\Upsilon$.
\end{definition}

In the case $\delta = 0$, we simply speak of clipped regularized empirical risk minimization (CR-ERM). In this case, $f_{D_n, \Upsilon}$ in fact can be also defined as follows:
\begin{align*}
 f_{D_n, \Upsilon} = \argmin_{f \in \mathcal{F}} \Upsilon(f) + \mathcal{R}_{L, D_n} (f).
\end{align*}
Note that on the right-hand side of \eqref{deltaCRERM} the unclipped loss is considered, and hence CR-ERMs do not necessarily minimize the 
regularized clipped empirical risk $\Upsilon(\cdot) + \mathcal{R}_{L,D_n}(\wideparen{\cdot})$. Moreover, in general CR-ERMs do not minimize
the regularized risk $\Upsilon(\cdot) + \mathcal{R}_{L,D_n}(\cdot)$ either, because on the left-hand side of \eqref{deltaCRERM} 
the clipped function is considered. However, if we have a minimizer of the unclipped regularized risk, then it automatically satisfies \eqref{deltaCRERM}. In particular, ERM decision functions satisfy \eqref{deltaCRERM} for the regularizer $\Upsilon := 0$ and $\delta : = 0$, 
and SVM decision functions satisfy \eqref{deltaCRERM} for the regularizer $\Upsilon := \lambda \|\cdot\|_H^2$ and $\delta := 0$. In other words, ERM and SVMs are CR-ERMs.

\section{Mixing Processes and A Generalized Bernstein-type Inequality}\label{GeneralizedBernstein}
In this section, we introduce a generalized Bernstein-type inequality. Here the inequality is said to be \textit{generalized} in that it depends on the effective number of observations instead of the number of observations, which, as we shall see later, makes it applicable to various stationary stochastic processes. To this end, let us first introduce several mixing processes.

\subsection{Several Stationary Mixing Processes}
We begin with introducing some notations. Recall that $(X,\mathcal{X})$ is a measurable space and $Y \subset \mathbb{R}$ is closed. We further denote $(\Omega, \mathcal{A}, \mu)$ as a probability space, $\mathcal{Z} := (Z_i)_{i \geq 1}$ as an $X \times Y$-valued stochastic process on $(\Omega, \mathcal{A}, \mu)$, 
$\mathcal{A}_1^i$ and $\mathcal{A}_{i+n}^{\infty}$ as the $\sigma$-algebras generated by $(Z_1, \ldots, Z_i)$ and $(Z_{i+n}, Z_{i+n+1}, \ldots)$, respectively. Throughout, we assume that $\mathcal{Z}$ is stationary, that is, the $(X \times Y)^n$-valued random variables $(Z_{i_1}, \ldots, Z_{i_n})$ and $(Z_{i_1+i}, \ldots, Z_{i_n+i})$ have the same distribution for all $n$, $i$, $i_1, \ldots, i_n \geq 1$. Let $\chi:\Omega\rightarrow X$ be a measurable map. $\mu_\mathsmaller{\chi}$ is denoted as the $\chi$-image measure of $\mu$, which is defined as $\mu_\mathsmaller{\chi}(B):= \mu(\chi^{-1}(B))$, $B\subset X$ measurable. We denote $L_p(\mu)$ as  the space of (equivalence classes of) measurable functions $g : \Omega \to \mathbb{R}$ with finite $L_p$-norm $\|g	\|_p$. Then $L_p(\mu)$ together with $\|g\|_p$ forms a Banach space. Moreover, if  $\mathcal{A}'\subset \mathcal{A}$ is a sub-$\sigma$-algebra, then $L_p(\mathcal A', \mu)$ denotes the space of all $\mathcal{A}'$-measurable functions $g\in L_p(\mu)$. $\ell_2^d$ denotes the space of $d$-dimensional sequences with finite Euclidean norm. Finally, for a Banach space $E$, we write $B_E$ for its closed unit ball.

In order to characterize the mixing property of a stationary stochastic process, various notions have been introduced in the literature \cite{Bradley05a}. Several frequently considered examples are $\alpha$-mixing, $\beta$-mixing and $\phi$-mixing, which are, respectively, defined as follows:  

\begin{definition}[\bf $\alpha$-Mixing Process]
A stochastic process $\mathcal{Z} = (Z_i)_{i \geq 1}$ is called \textbf{$\alpha$-mixing} if there holds
\begin{align*}
\lim_{n \to \infty} \alpha(\mathcal{Z}, n) = 0,
\end{align*}
where $\alpha(\mathcal{Z}, n)$ is the $\alpha$-mixing coefficient defined by
\begin{align*}
\alpha(\mathcal{Z}, n) 
= \sup_{A \in \mathcal{A}_1^i,\, B \in \mathcal{A}_{i+n}^{\infty}} | \mu(A \cap B) - \mu(A) \mu(B) |.
\end{align*}
Moreover, a stochastic process $\mathcal{Z}$ is called \textbf{geometrically $\alpha$-mixing}, if 
\begin{align}\label{AlphaExpDecay}
\alpha(\mathcal{Z}, n) \leq c \exp (- b n^{\gamma}), \,\,\,\,\,\,\,\,\,\, n \geq 1, 
\end{align}
for some constants $b > 0$, $c \geq 0$, and $\gamma > 0$.
\end{definition}

\begin{definition}[\bf $\beta$-Mixing Process]
A stochastic process $\mathcal{Z} = (Z_i)_{i \geq 1}$ is called \textbf{$\beta$-mixing} if there holds
\begin{align*}
\lim_{n \to \infty} \beta(\mathcal{Z}, n) = 0,
\end{align*}
where $\beta(\mathcal{Z}, n)$ is the $\beta$-mixing coefficient defined by
\begin{align*}
\beta(\mathcal{Z}, n) := \mathbb{E} \sup_{B \in \mathcal{A}_{i+n}^{\infty}} | \mu(B) - \mu(B|\mathcal{A}_1^i) |.
\end{align*}
\end{definition}

\begin{definition}[\bf $\phi$-Mixing Process]
A stochastic process $\mathcal{Z} = (Z_i)_{i \geq 1}$ is called \textbf{$\phi$-mixing} if there holds
\begin{align*}
\lim_{n \to \infty} \phi(\mathcal{Z}, n) = 0,
\end{align*}
where $\phi(\mathcal{Z}, n)$ is the $\phi$-mixing coefficient defined by
\begin{align*}
\phi(\mathcal{Z}, n) := \sup_{A \in \mathcal{A}_1^i, B \in \mathcal{A}_{i+n}^{\infty}} | \mu(B) - \mu(B|A) |. 
\end{align*}
\end{definition}
 
The $\alpha$-mixing concept was introduced by Rosenblatt \cite{Rosenblatt56a} while the $\beta$-mixing coefficient was introduced by \cite{VoRo59a,VoRo59b}, and was attributed there to Kolmogorov. Moreover, Ibragimov \cite{Ibragimov62a} introduced the  $\phi$-coefficient, see also \cite{IbRo79a}. An extensive and thorough account on mixing concepts including $\beta$- and $\phi$-mixing is also provided by \cite{Bradley07}. It is well-known that, see e.g.~\cite[Section 2]{HaSt14a}, the $\beta$- and $\phi$-mixing sequences are also $\alpha$-mixing, see Figure \ref{Relationship1Alpha}. From the above definition, it is obvious that i.i.d.~processes are also geometrically $\alpha$-mixing processes since \eqref{AlphaExpDecay} is satisfied for $c = 0$ and all $b, \gamma > 0$. Moreover, several time series models such as ARMA and GARCH, which are often used to describe, e.g.~financial data, satisfy \eqref{AlphaExpDecay} under natural conditions \cite[Chapter 2.6.1]{FaYa03a}, and the same is true for many Markov chains including some dynamical systems perturbed by dynamic noise, see e.g. \cite[Chapter 3.5]{Vidyasagar03a}.

Another important class of mixing processes called (time-reversed) $\mathcal{C}$-mixing processes was originally introduced in \cite{Maume06a} and recently investigated in \cite{HaSt15a}. As shown below, it is defined in association with a function class that takes into account of the smoothness of functions and therefore could be more general in the dynamical system context. As illustrated in \cite{Maume06a} and \cite{HaSt15a}, the $\mathcal{C}$-mixing process encounters a large family of dynamical systems. Given a semi-norm $\|\cdot\|$ on a vector space $E$ of bounded measurable functions $f : Z \to \mathbb{R}$, we define the \textbf{$\mathcal{C}$-norm} by
\begin{align}\label{lambdanorm}
\|f\|_{\mathcal{C}} := \|f\|_{\infty} + \|f\|,
\end{align} 
and denote the space of all bounded $\mathcal{C}$-functions by 
$\mathcal{C}(Z) := \left\{ f : Z \to \mathbb{R} \, \big| \, \|f\|_{\mathcal{C}} < \infty \right\}$.

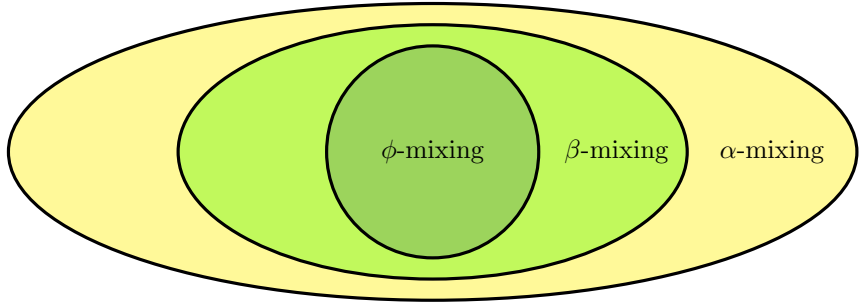
\begin{figure}
\begin{center}
\begin{tikzpicture}[scale = 0.94]
  \begin{scope}[fill opacity=0.4]
    \fill[blue]  \firstcircle;
    \fill[green] \thirdellipse;
    \fill[yellow] \fourthellipse;
  \end{scope}
  \begin{scope}[very thick,font=\large]
    \draw \firstcircle node at (6,-4) {\small{$\phi$-mixing}};
    \draw \thirdellipse node at (8.6,-4) {\small{$\beta$-mixing}};
    \draw \fourthellipse node at (10.8,-4) {\small{$\alpha$-mixing}};
  \end{scope}
\end{tikzpicture}
\caption{Relations among $\alpha$-, $\beta$-, and $\phi$-mixing processes}
\label{Relationship1Alpha}
\end{center}
\end{figure}

\begin{figure}
\begin{center}
\begin{tikzpicture}[scale = 0.94]
  \begin{scope}[fill opacity=0.4]
    \fill[blue]  \firstcircle;
    \fill[green] \firstellipse;
    \fill[yellow] \secondellipse;
  \end{scope}
  \begin{scope}[very thick,font=\large]
    \draw \firstcircle node at (6,-4) {\small{$\phi$-mixing}};
    \draw \firstellipse node at (1,-4) {\small{$\alpha$-mixing}};
    \draw \secondellipse node at (11,-4) {\small{$\mathcal{C}$-mixing}};
  \end{scope}
\end{tikzpicture}
\caption{Relations among $\alpha$-, $\phi$-, and $\mathcal{C}$-mixing processes}
\label{Relationship2}
\end{center}
\end{figure}
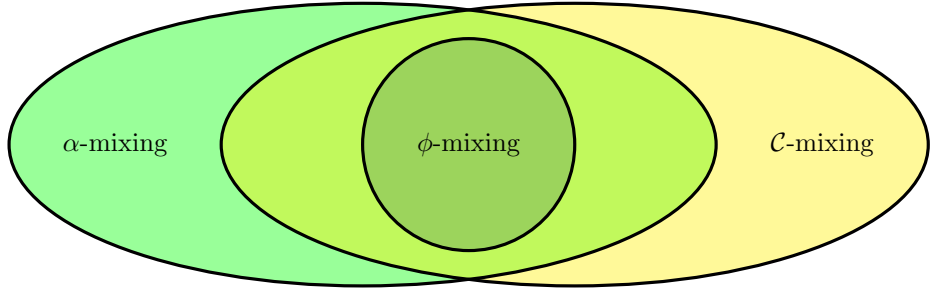

\begin{definition}[\bf $\mathcal{C}$-Mixing Process]\label{definition::c_mixing}
Let $\mathcal{Z} = (Z_i)_{i \geq 1}$ be a stationary stochastic process. For $n \geq 1$, the $\mathcal{C}$-mixing coefficients are defined by
\begin{align*} 
\phi_{\mathcal{C}}(\mathcal{Z}, n) := \sup \big \{ \mathrm{cor}(\psi, h \circ Z_{k+n}) : k \geq 1,\,\psi \in B_{L_1(\mathcal{A}_1^k, \mu)}, h \in B_{\mathcal{C}(Z)} \big\},
\end{align*} 
and similarly, the time-reversed $\mathcal{C}$-mixing coefficients are defined by
\begin{align*} 
\phi_{\mathcal{C}, \text{rev}}(\mathcal{Z}, n) := \sup \big\{ \mathrm{cor}(h \circ Z_k, \varphi) : k \geq 1, h \in B_{\mathcal{C}(Z)}, \varphi \in B_{L_1(\mathcal{A}_{k+n}^{\infty}, \mu)}  \big\}.
\end{align*} 
Let $(d_n)_{n \geq 1}$ be a strictly positive sequence converging to $0$. Then we say that $\mathcal{Z}$ is \textbf{(time-reversed) $\mathcal{C}$-mixing} with rate $(d_n)_{n \geq 1}$, if we have $\phi_{\mathcal{C},(\text{rev})}(\mathcal{Z}, n) \leq d_n$ for all $n \geq 1$. Moreover, if $(d_n)_{n \geq 1}$ is of the form
\begin{align*}   
d_n := c \exp \bigl( - b n^{\gamma} \bigr), ~~~~~~ n \geq 1, 
\end{align*}
for some constants $c > 0$, $b > 0$, and $\gamma > 0$, then $\mathcal{X} $ is called  \textbf{ geometrically  (time-reversed) $\mathcal{C}$-mixing}. If $(d_n)_{n \geq 1}$ is of the form
\begin{align*}   
d_n := c\cdot n^{-\gamma}, ~~~~~~ n \geq 1, 
\end{align*}
for some constants $c > 0$, and $\gamma > 0$, then $\mathcal{X} $ is called  \textbf{ polynomial (time-reversed) $\mathcal{C}$-mixing}.
\end{definition}

Figure \ref{Relationship2} illustrates the relations among $\alpha$-mixing processes, $\phi$-mixing processes, and $\mathcal{C}$-mixing processes. Clearly, $\phi$-mixing processes are $\mathcal{C}$-mixing  \cite{HaSt15a}. Furthermore, various  discrete-time dynamical systems including Lasota-Yorke maps, uni-modal maps, and piecewise expanding maps in higher dimension are $\mathcal{C}$-mixing, see \cite{Maume06a}. Moreover, smooth expanding maps on manifolds, piecewise expanding maps, uniformly hyperbolic attractors, and non-uniformly hyperbolic uni-modal maps are time-reversed geometrically $\mathcal{C}$-mixing, see \cite[Proposition 2.7, Proposition 3.8, Corollary 4.11 and Theorem 5.15]{Viana97a}, respectively.

\subsection{A Generalized Bernstein-type Inequality}
As discussed in the introduction, the Bernstein-type inequality plays an important role in many areas of probability and statistics. In the statistical learning theory literature, it is also crucial in conducting concentrated estimation for learning schemes. As mentioned previously, these inequalities are usually  presented in rather complicated forms under different assumptions, which therefore limit their portability to other contexts. However, what is common behind these inequalities is their relying on the boundedness assumption of the variance. Given the above discussions, in this subsection, we introduce the following generalized Bernstein-type inequality, with the hope of making it as an off-the-shelf tool for various mixing processes.   

\begin{assumption}\label{BernsteinInequalityGeneral}
Let $\mathcal{Z} := (Z_i)_{i \geq 1}$ be an $X \times Y$-valued, stationary stochastic process on $(\Omega, \mathcal{A}, \mu)$ and $P := \mu_{Z_1}$. Furthermore, let $h : X \times Y \to \mathbb{R}$ be a bounded measurable function for which there exist constants $B > 0$ and $\sigma \geq 0$ such that $\mathbb{E}_P h = 0$, $\mathbb{E}_P h^2 \leq \sigma^2$, and $\|h\|_{\infty} \leq B$. Assume that, for all $\varepsilon > 0$, there exist constants $n_0 \geq 1$ independent of $\varepsilon$ and $\neff \geq 1$ such that for all $n \geq n_0$, we have
\begin{align}\label{BersteinGeneral}
P \left( \frac{1}{n} \sum_{i=1}^n h(Z_i) \geq \varepsilon \right)\leq C \exp \left( - \frac{\varepsilon^2 \neff}{c_\sigma \sigma^2 + c_{\mathsmaller{B}} \varepsilon B} \right),
\end{align}
where $\neff\leq n$ is the effective number of observations, $C$ is a constant independent of $n$, and $c_\sigma  $, $c_{\mathsmaller{B}}  $ are positive constants.
\end{assumption}

Note that in Assumption \ref{BernsteinInequalityGeneral}, the generalized Bernstein-type inequality \eqref{BersteinGeneral} is assumed with respect to $\neff$ instead of $n$, which is a function of $n$ and is termed as the \textit{effective number of observations}. The terminology, \textit{effective number of observations}, ``provides a heuristic understanding of the fact that the statistical properties of autocorrelated data are similar to a suitably defined number of independent observations" \cite{zikeba2010effective}. We will continue our discussion on the effective number of observations $\neff$ in Subsection \ref{subsec::effective_observations} below.

\subsection{Instantiation to Various Mixing Processes} 
We now show that the generalized Bernstein-type inequality in Assumption \ref{BernsteinInequalityGeneral} can be instantiated to various mixing processes, e.g., i.i.d processes, geometrically $\alpha$-mixing processes, restricted geometrically $\alpha$-mixing processes, geometrically $\alpha$-mixing Markov chains, $\phi$-mixing processes, geometrically $\mathcal{C}$-mixing processes, polynomially $\mathcal{C}$-mixing processes, among others. 

\subsubsection{I.I.D Processes}
Clearly, the classical Bernstein inequality \cite{Bernstein1946a} satisfies \eqref{BersteinGeneral} with $n_0 = 1$, $C = 1$, $c_\sigma = 2$,  $c_{\mathsmaller{B}} = 2/3$, and $\neff=n$. 

\subsubsection{Geometrically $\alpha$-Mixing Processes}
For stationary geometrically $\alpha$-mixing processes $\mathcal{Z}$,
\cite[Theorem 4.3]{MoMa96a}
bounds the left-hand side of \eqref{BersteinGeneral} by
\begin{align*}   
(1 + 4 e^{-2} c) \exp \biggl( - \frac{3 \varepsilon^2 n^{{(\gamma)}}}{6 \sigma^2 + 2 \varepsilon B} \biggr)
\end{align*}
for any $n \geq 1$, and $\varepsilon > 0$, where 
\begin{align*}
n^{{(\gamma)}} := \Bigl\lfloor n \bigl\lceil ( 8n/b )^{1/(\gamma + 1)} \bigr\rceil^{-1} \Bigr\rfloor,
\end{align*}
where $\lfloor t \rfloor$ is the largest integer less than or equal to $t$ and $\lceil t \rceil$ is the smallest integer greater than or equal to $t$ for $t \in \mathbb{R}$. Observe that $\lceil t \rceil \leq 2 t$ for all $t \geq 1$ and $\lfloor t \rfloor \geq t/2$ for all $t \geq 2$. From this it is easy to conclude that, for all $n\geq n_0$ with  
\begin{align}\label{geometrical_alpha_n0}
n_0 := \max \{ b/8, 2^{2+5/\gamma} b^{-1/\gamma} \},
\end{align}
we have $n^{{(\gamma)}} \geq 2^{- \frac{2 \gamma + 5}{\gamma + 1}} b^{\frac{1}{\gamma + 1}} n^{\frac{\gamma}{\gamma + 1}}$.
Hence, the right-hand side of \eqref{BersteinGeneral} takes the form
\begin{align*}  
(1 + 4 e^{-2} c) \exp \biggl( - \frac{\varepsilon^2 n^{\gamma/(\gamma + 1)}}{( 8^{2+\gamma}/b )^{1/(1+\gamma)} ( \sigma^2 + \varepsilon B/3 )} \biggr).
\end{align*}
It is easily seen that this bound is of the generalized form \eqref{BersteinGeneral} with $n_0$ given in \eqref{geometrical_alpha_n0}, $C = 1 + 4 e^{-2} c$, $c_\sigma = ( 8^{2+\gamma}/b )^{1/(1+\gamma)}$, $c_{\mathsmaller{B}} = ( 8^{2+\gamma}/b )^{1/(1+\gamma)}/3$, and $\neff = n^{\gamma/(\gamma + 1)}$.

\subsubsection{Restricted Geometrically $\alpha$-Mixing Processes}
A restricted geometrically $\alpha$-mixing process is referred to as a geometrically $\alpha$-mixing process  (see Definition \ref{definition::c_mixing})  with $\gamma \geq 1$. For this kind of $\alpha$-mixing processes, \cite[Theorem 2]{MePeRi09a} established a bound for the right-hand side of \eqref{BersteinGeneral} that takes the following form
\begin{align}\label{bernsteininequalityMePeRi1Alpha}
c_c \exp \biggl( - \frac{c_b \varepsilon^2 n}{v^2 + B^2/n + \varepsilon B (\log n)^2} \biggr),
\end{align}
for all $\varepsilon > 0$ and $n \geq 2$, where $c_b$ is some constant depending only on $b$, $c_c$ is some constant depending only on $c$, 
and $v^2$ is defined by
\begin{align}\label{vquadrat}
 v^2 := \sigma^2 + 2 \sum_{2 \leq i \leq n} |\mathrm{cov}(h(X_1), h(X_i))|\, .
\end{align}

In fact, for any $\epsilon > 0$, by using Davydov's covariance inequality \cite[Corollary to Lemma 2.1]{Davydov68a} with 
$p = q = 2 + \epsilon$ and $r = (2+\epsilon)/\epsilon$, 
we obtain for $i \geq 2$,
\begin{align*}
\mathrm{cov}(h(Z_1), h(Z_i))
& \leq 8 \|h(Z_1)\|_{2+\epsilon} \|h(Z_i)\|_{2+\epsilon}\alpha(\mathcal{Z}, i-1)^{\epsilon/(2+\epsilon)}
\\
& \leq 8 \bigl( \mathbb{E}_P h^{2+\epsilon} \bigr)^{2/(2+\epsilon)}\bigl( c e^{-b(i-1)} \bigr)^{\epsilon/(2+\epsilon)}
\\
& \leq 8 c^{\epsilon/(2+\epsilon)} B^{2\epsilon/(2+\epsilon)} \sigma^{2 \cdot 2/(2+\epsilon)}\exp \bigl( - b \epsilon (i-1) /(2+\epsilon)  \bigr).
\end{align*}
Consequently, we have
\begin{align*}
v^2 
& \leq \sigma^2 + 16 c^{\epsilon/(2+\epsilon)} B^{2\epsilon/(2+\epsilon)} \sigma^{2 \cdot 2/(2+\epsilon)}\sum_{2 \leq i \leq n} \exp \bigl( - b \epsilon(i-1)/(2+\epsilon)  \bigr)
\\
& \leq \sigma^2 + 16 c^{\epsilon/(2+\epsilon)} B^{2\epsilon/(2+\epsilon)} \sigma^{2 \cdot 2/(2+\epsilon)}\sum_{i \geq 1} \exp ( - b \epsilon i / (2+\epsilon) ).
\end{align*}
Setting
\begin{align}\label{CzetaAlpha}
 c_{\epsilon} := 16 c^{\epsilon/(2+\epsilon)} B^{2\epsilon/(2+\epsilon)} 
 \sum_{i \geq 1} \exp ( - b \epsilon i/ (2+\epsilon) ),
\end{align}
then the probability bound \eqref{bernsteininequalityMePeRi1Alpha} can be reformulated as 
\begin{align*}  
c_c \exp \biggl( - \frac{c_b \varepsilon^2 n}{\sigma^2 + c_{\epsilon} \sigma^{4/(2+\epsilon)} + B^2/n + \varepsilon B (\log n)^2} \biggr).
\end{align*}
When $ n \geq n_0$ with 
\begin{align*}
 n_0 := \max \left\{ 3, \exp \bigl( \sqrt{c_{\epsilon}} \sigma^{- \epsilon/(2+\epsilon)} \bigr),B^2/\sigma^2 \right\},
\end{align*}
it can be  further upper bounded by
\begin{align*}
c_c \exp \left( - \frac{c_b \varepsilon^2 (n/(\log n)^2)}{3 \sigma^2 + \varepsilon B } \right).
\end{align*}
Therefore, the Bernstein-type inequality for the restricted $\mathcal{C}$-mixing process is also of the generalized form \eqref{BersteinGeneral} where $C = c_c$, $c_\sigma = 3/c_b$, $c_{\mathsmaller{B}} = 1/c_b$, and $\neff = n/(\log n)^2$.

\subsubsection{Geometrically $\alpha$-Mixing Markov Chains}
For the stationary geometrically $\alpha$-mixing Markov chain with centered and bounded random variables, \cite{Adamczak08a} bounds the  left-hand side of \eqref{BersteinGeneral} by
\begin{align*}
\exp \biggl( - \frac{n \varepsilon^2}{\tilde \sigma^2 + \varepsilon B \log n} \biggr),
\end{align*}
where $\tilde \sigma^2 = \lim_{n \to \infty} \frac{1}{n}\cdot \mathrm{Var} \sum_{i=1}^n h(X_i)$.

Following the similar arguments as in  the restricted geometrically $\mathcal{C}$-mixing case, we know that for an arbitrary $\epsilon > 0$, there holds
\begin{align*}
\mathrm{Var} \sum_{i=1}^n h(X_i)
& = n \sigma^2 + 2 \sum_{1 \leq i < j \leq n} |\mathrm{cov}(h(X_i), h(X_j))|
\\
& \leq n \sigma^2 + 2 n \sum_{2 \leq i \leq n} |\mathrm{cov}(h(X_1), h(X_i))|
\\
& = n \cdot v^2 \leq n \bigl( \sigma^2 + c_{\epsilon} \sigma^{4/(2+\epsilon)} \bigr),
\end{align*}
where $v^2$ is defined in \eqref{vquadrat} and $c_{\epsilon}$ is given in \eqref{CzetaAlpha}.  Consequently the following inequality holds
\begin{align*}   
\exp \biggl( - \frac{n \varepsilon^2}{\sigma^2 + c_{\epsilon} \sigma^{4/(2+\epsilon)} + \varepsilon B \log n} \biggr) \leq \exp \biggl( - \frac{\varepsilon^2 (n/\log n)} {2 \sigma^2 + \varepsilon B } \biggr),
\end{align*}
for $ n \geq n_0$ with 
\begin{align*}
n_0 := \max \Bigl\{ 3, \exp \bigl( c_{\epsilon} \sigma^{-2\epsilon/(2+\epsilon)} \bigr) \Bigr\}.
\end{align*}
That is, when $n \geq n_0$ the Bernstein-type inequality for the geometrically $\alpha$-mixing Markov chain can be also formulated as the generalized form \eqref{BersteinGeneral} with $C = 1$, $c_\sigma = 2$, $c_{\mathsmaller{B}} = 1$, and $\neff=n/\log n$.

\subsubsection{$\phi$-Mixing Processes}\label{subsubsection phi}
For a $\phi$-mixing process $\mathcal{Z}$, \cite{Samson00a} provides the following bound for the left-hand side of \eqref{BersteinGeneral}  
\begin{align*} 
\exp \biggl( - \frac{\varepsilon^2 n}{8 c_{\phi} (4 \sigma^2 + \varepsilon B)} \biggr),
\end{align*}
where $c_{\phi} := \sum_{k=1}^{\infty} \sqrt{\phi(\mathcal{Z}, k)}$. Obviously, it is of the general form \eqref{BersteinGeneral} with $n_0 = 1$, $C = 1$, $c_\sigma = 32 c_{\phi}$, $c_{\mathsmaller{B}} = 8 c_{\phi}$, and $\neff=n$.

\subsubsection{Geometrically $\mathcal{C}$-Mixing Processes}
For the geometrically $\mathcal{C}$-mixing process in Definition \ref{definition::c_mixing}, \cite{HaSt15a} recently developed a Bernstein-type inequality. To state the inequality, the following assumption on the semi-norm $\|\cdot\|$ in \eqref{lambdanorm} is needed 
\begin{align*}
\bigl\| e^f \bigr\|  \leq  \bigl\| e^f \bigr\|_{\infty} \|f\|,  \,\,\,\,\,\,\,  f \in \mathcal{C}(Z).
\end{align*}
Under the above restriction on the semi-norm $\|\cdot\|$ in \eqref{lambdanorm}, and the assumptions that $\|h\|\leq A$, $\|h\|_{\infty} \leq B$, and $\mathbb E_P h^2 \leq \sigma^2$, \cite{HaSt15a} states that when $n\geq n_0$ with  	 
\begin{align}\label{nzero}
n_0 := \max \left\{ \min \bigl\{ m \geq 3 : m^2 \geq 808 c (3A + B)/B \text{ and } m/(\log m)^{2/\gamma} \geq 4 \bigr\}, e^{3/b} \right\}.
\end{align}
the right-hand side of \eqref{BersteinGeneral} takes the form
\begin{align}\label{bernsteininequalityCMixing}
2 \exp \biggl( - \frac{\varepsilon^2 n / (\log n)^{2/\gamma}}{8  (\sigma^2 + \varepsilon B/3)} \biggr).
\end{align}
It is easy to see that \eqref{bernsteininequalityCMixing} is also of the generalized form \eqref{BersteinGeneral} with $n_0$ given in \eqref{nzero}, $C = 2$, $c_\sigma = 8$, $c_{\mathsmaller{B}} = 8/3$, and $\neff = n / (\log n)^{2/\gamma}$.

\subsubsection{Polynomially $\mathcal{C}$-Mixing Processes}
For the polynomially $\mathcal{C}$-mixing processes, a Bernstein-type inequality was established recently in \cite{hanglearning}. Under the same restriction on the semi-norm $\|\cdot\|$ and assumption on $h$ as in the geometrically $\mathcal{C}$-mixing case, it states that when $n\geq n_0$ with
\begin{align}\label{ploynomial_c_n_0}
n_0 :=\max\Bigl\{\bigl(808c(3A+B)/B\bigr)^{1/2}, \, 4^{(\gamma+1)/(\gamma-2)}\Bigr\}, \,\, \gamma > 2,
\end{align}
the right-hand side of \eqref{BersteinGeneral} takes the form
\begin{align*} 
2 \exp \biggl( - \frac{\varepsilon^2 n^{(\gamma-2)/(\gamma+1)}}{8  (\sigma^2 + \varepsilon B/3)} \biggr).
\end{align*}
An easy computation shows that it is also of the generalized form \eqref{BersteinGeneral} with $n_0$ given in \eqref{ploynomial_c_n_0}, $C = 2$, $c_\sigma = 8$, $c_{\mathsmaller{B}} = 8/3$, and $\neff = n^{(\gamma-2)/(\gamma+1)}$ with $\gamma>2$.

\subsection{From Observations to Effective Observations}\label{subsec::effective_observations}
The generalized Bernstein-type inequality in Assumption \ref{BernsteinInequalityGeneral} is assumed with respect to the\textit{ effective number of observations} $\neff$. As verified above, the assumed generalized Bernstein-type inequality  indeed holds for many   mixing processes whereas $\neff$ may take different values in different circumstances. Supposing that we have $n$ observations drawn from a certain  mixing process discussed above, Table \ref{Eff.No.} reports its effective number of observations. As mentioned above, it can be roughly treated as the number of independent observations when inferring the statistical properties of correlated data. In this subsection, we make some effort in presenting an intuitive understanding towards the meaning of the effective number of observations. 
\begin{table}[h] 
 \setlength{\tabcolsep}{9pt}
 \renewcommand{\arraystretch}{1.5}
\begin{center}
    \caption{Effective Number of Observations for Different Mixing Processes}
  \begin{tabular}{ l | l }
    \hline
    \emph{examples} & \emph{effective number of observations} \\ \hline
       \text{i.i.d processes} & $n $ \\  \hline  
   \text{geometrically $\alpha$-mixing processes} & $n^{\gamma / (\gamma+1)}$ \\  \hline
   \text{restricted geometrically $\alpha$-mixing processes} & $n/(\log n)^2$ \\ \hline
    \text{geometrically $\alpha$-mixing Markov chains} & $n/\log n$ \\ \hline
    \text{$\phi$-mixing processes} & $n$ \\ \hline
    \text{geometrically $\mathcal{C}$-mixing processes} & $n/(\log n)^{2/\gamma}$ \\ \hline
    \text{polynomially $\mathcal{C}$-mixing processes} & $n^{(\gamma-2)/(\gamma+1)}$ with $\gamma>2$ \\ \hline
  \end{tabular}
    \label{Eff.No.}
\end{center}
\end{table}

The terminology - \textit{effective observations}, which may be also referred as the \textit{effective number of observations} depending on the context, appeared probably first in \cite{bayley1946effective} when studying the autocorrelated time series data. In fact, many similar concepts can be found in the literature of statistical learning from mixing processes, see e.g., \cite{lubman1969spatial,yaglom1987correlation,MoMa96a,csen1998small,zikeba2010effective}. For stochastic processes, \textit{mixing} indicates the \textit{asymptotic independence}. In some sense, the effective observations can be taken as the independent observations that can contribute when learning from a certain mixing process. 

\medskip
\medskip

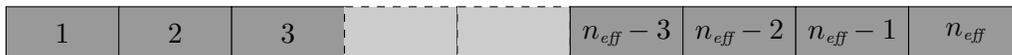
\begin{figure}[h]
\begin{center}
\begin{tikzpicture} 
\draw[fill=black!40!white, solid] (0,0) rectangle (1.5,0.7)node[pos=.5] {$1$};
\draw[fill=black!40!white, solid] (1.5,0) rectangle (3.0,0.7)node[pos=.5] {$2$};
\draw[fill=black!40!white, solid] (3.0,0) rectangle (4.5,0.7)node[pos=.5] {$3$};
\draw[fill=black!20!white, dashed] (4.5,0) rectangle (6.0,0.7);
\draw[fill=black!20!white, dashed] (6.0,0) rectangle (7.5,0.7);
\draw[fill=black!40!white, solid] (7.5,0) rectangle (9.0,0.7)node[pos=.5] {$\neff-3$};
\draw[fill=black!40!white, solid] (9.0,0) rectangle (10.5,0.7)node[pos=.5] {$\neff-2$};
\draw[fill=black!40!white, solid] (10.5,0) rectangle (12.0,0.7)node[pos=.5] {$\neff-1$};
\draw[fill=black!40!white, solid] (12.0,0) rectangle (13.5,0.7)node[pos=.5] {$\neff$};
\end{tikzpicture}
\end{center}
\caption{An illustration of the effective number of observations when inferring the statistical properties of the data drawn from mixing processes. The data of size $n$ are split into $\neff$ blocks, each of size $n/\neff$.}\label{block_illustration}
\end{figure}

In fact, when inferring statistical properties with data drawn from mixing processes, a frequently employed technique is to split the data of size $n$ into $k$ blocks, each of size $\ell$ \cite{Yu94a,MoMa96a,bosq2012nonparametric,MoRo09b,HaSt15a}. Each block may be constructed either by choosing consecutive points in the original observation set or by a jump selection \cite{MoMa96a, HaSt15a}. With the constructed blocks, one can then introduce a new sequence of blocks that are independent between the blocks by using the coupling technique. Due to the mixing assumption, the difference between the two sequences of blocks can be measured with respect to a certain metric. Therefore, one can deal with the independent  blocks instead of dependent blocks now. On the other hand, for observations in each originally constructed block, one can again apply the coupling technique \cite{bosq2012nonparametric,duchi2012ergodic} to tackle, e.g., introducing $\ell$ new i.i.d observations and bounding the difference between the newly introduced observations and the original   observations with respect to a certain metric. During this process, one tries to ensure that the number of blocks $k$ is as large as possible, for which $\neff$ turns out to be the choice. An intuitive illustration of this procedure is shown in Fig.\,\ref{block_illustration}.

\section{A Generalized Sharp Oracle Inequality} \label{MainResults}

In this section we present one of our main results: an oracle inequality for learning from mixing processes satisfying the generalized Bernstein-type inequality \eqref{BersteinGeneral}. We first introduce a few more notations. Let $\mathcal{F}$ be a hypothesis set in the sense of Definition \ref{crerm}. For 
\begin{align}\label{rstar}
r^* := \inf_{f \in \mathcal{F}} \Upsilon(f) + \mathcal{R}_{L,P}(\wideparen{f} \,) - \mathcal{R}_{L,P}^*,
\end{align}
and $r > r^*$, we write
\begin{align}\label{Fr}
\mathcal{F}_r := \left\{ f \in \mathcal{F} : \Upsilon(f) + \mathcal{R}_{L,P}(\wideparen{f}\,) - \mathcal{R}_{L,P}^* \leq r \right\}.
\end{align}
Since $L(x, y, 0) \leq 1$, $0 \in \mathcal{F}$, and $\Upsilon(0) = 0$,
then we have $r^* \leq 1$, 
Furthermore, we assume that there exists 
a function $\varphi: (0, \infty) \to (0, \infty)$ that satisfies
\begin{align}\label{coveringnumber}
\ln \mathcal{N}(\mathcal{F}_r, \|\cdot\|_{\infty}, \varepsilon) \leq \varphi(\varepsilon) r^p
\end{align}
for all $\varepsilon > 0$, $r > 0$ and a suitable constant $p \in (0, 1]$. Note that there are actually many hypothesis sets  satisfying Assumption \eqref{coveringnumber}, see Section \ref{applications} for some examples.

Now, we present the oracle inequality as follows:

\begin{thm}\label{oracleinequality}
Let $\mathcal{Z}$ be a stochastic process satisfying Assumption \ref{BernsteinInequalityGeneral} with constants $n_0 \geq 1$, $C > 0$, $c_\sigma> 0$, and $c_{\mathsmaller{B}} > 0$.  Furthermore, let $L$ be a loss satisfying Assumption \ref{assumptionL}.  Moreover, assume that there exists a Bayes decision function $f_{L, P}^*$ and constants $\vartheta \in [0, 1]$ and $V \geq 1$ such that 
\begin{align}\label{variancebound}
\mathbb{E}_P (L \circ \wideparen{f} - L \circ f_{L, P}^*)^2 \leq V \cdot \left( \mathbb{E}_P (L \circ \wideparen{f} - L \circ f_{L, P}^*) \right)^{\vartheta},
\,\,\,\,\,\,\,\,\, 
f \in \mathcal{F}, 
\end{align}
where $\mathcal{F}$ is a hypothesis set with $0 \in \mathcal{F}$. We define $r^*$ and $\mathcal{F}_r$ by \eqref{rstar} and \eqref{Fr}, respectively and assume that \eqref{coveringnumber} is satisfied. Finally, let $\Upsilon : \mathcal{F} \to [0, \infty)$ be a regularizer with $\Upsilon(0) = 0$, $f_0 \in \mathcal{F}$ be a fixed function, and $B_0 \geq 1$ be a constant such that $\| L \circ f_0 \|_{\infty} \leq B_0$. Then, for all fixed $\varepsilon > 0$, $\delta \geq 0$, $\tau \geq 1$, $n \geq n_0$, and $r \in (0, 1]$ satisfying
\begin{align}\label{minradius}
r \geq \max \left\{ 
\biggl( \frac{c_{\mathsmaller{V}} (\tau + \varphi(\varepsilon/2) 2^p r^p)}{\neff} \biggr)^{\frac{1}{2 - \vartheta}}, \frac{8 c_{\mathsmaller{B}} B_0 \tau}{\neff}, r^* \right\}
\end{align}
with $c_{\mathsmaller{V}} := 64 (4 c_{\sigma} V + c_{\mathsmaller{B}})$, every learning method defined by \eqref{deltaCRERM} satisfies with 
probability $\mu$ not less than $1 - 8 C e^{-\tau}$:
\begin{align} \label{oracleinequalityy}
\Upsilon(f_{D_n,\Upsilon}) + \mathcal{R}_{L,P}(\wideparen{f}_{D_n, \Upsilon}) - \mathcal{R}_{L,P}^*< 2 \Upsilon(f_0) + 4 \mathcal{R}_{L, P}(f_0) - 4 \mathcal{R}_{L, P}^* + 4 r + 5 \varepsilon + 2 \delta.
\end{align}
\end{thm}

The proof of Theorem \ref{oracleinequality} will be provided in the Appendix. Before we illustrate this oracle inequality   
in the next section with various examples, let us briefly discuss the variance bound \eqref{variancebound}. For example, if $Y = [- M, M]$ and $L$ is the least squares loss, then it is well-known that \eqref{variancebound} is satisfied for $V := 16 M^2$ and $\vartheta = 1$, see e.g. \cite[Example 7.3]{StCh08a}. Moreover, under some assumptions on the distribution $P$, \cite{StCh11a} established a variance bound of the form \eqref{variancebound} for the so-called pinball loss used for quantile regression. In addition, for the hinge loss, \eqref{variancebound} is satisfied for $\vartheta := q / (q + 1)$, if Tsybakov's noise assumption \cite[Proposition 1]{Tsybakov04a} holds for $q$, see \cite[Theorem 8.24]{StCh08a}. Finally, based on \cite{BlLuVa03a}, \cite{Steinwart09a} established a variance bound with 
$\vartheta = 1$ for the earlier mentioned clippable modifications of strictly convex, twice continuously differentiable margin-based loss 
functions. 

We remark that in  Theorem \ref{oracleinequality} the constant $B_0$ is necessary since the assumed boundedness of $L$ only guarantees $\|L \circ \wideparen{f}\|_{\infty}\leq 1$, while $B_0$ bounds the function $L \circ f_0$ for an {\em unclipped} $f_0 \in \mathcal{F}$. We do not assume that all $f \in \mathcal{F}$ satisfy $\wideparen{f} = f$, therefore in general $B_0$ is necessary. We refer to Examples \ref{LSSVMfixedkernels}, \ref{ex-smooth-kernesl} and \ref{LSSVMgaussiankernels} for situations, where $B_0$ is significantly larger than $1$.

\section{Applications to Statistical Learning}\label{applications}
To illustrate the oracle inequality developed in Section \ref{MainResults}, we now apply it to establish learning rates for some algorithms including ERM over finite sets and SVMs using either a given generic kernel or a Gaussian kernel with varying widths. In the ERM case, our results match those in the i.i.d.~case, if one replaces the number of observations $n$ with the effective number of observations $\neff$ while, for LS-SVMs with given generic kernels, our rates are slightly worse than the recently obtained optimal rates \cite{StHuSc09b} for i.i.d.~observations. The latter difference is not surprising when considering the fact that \cite{StHuSc09b} used heavy machinery from 
empirical process theory such as Talagrand's inequality and localized Rademacher averages while our results only use a light-weight argument based on the generalized Bernstein-type inequality and the peeling method. However, when using  Gaussian kernels, we indeed recover the optimal rates for LS-SVMs and SVMs for quantile regression with i.i.d.~observations.

Let us now present the first example, that is, the empirical risk minimization scheme over a finite hypothesis set. 

\begin{example}[ERM]\label{ERM_example} 
Let $\mathcal{Z}$ be a stochastic process satisfying Assumption \ref{BernsteinInequalityGeneral}
and the hypothesis set $\mathcal{F}$ be finite with $0 \in \mathcal{F}$ and $\Upsilon(f) = 0$ 
for all $f \in \mathcal{F}$. Moreover, assume that $\|f\|_{\infty} \leq M$ for all $f \in \mathcal{F}$. 
Then, for accuracy $\delta := 0$, the learning method described by \eqref{deltaCRERM}
is ERM, and Theorem \ref{oracleinequality} shows by some simple estimates that 
\begin{align*}
\mathcal{R}_{L,P} (f_{D_n, \Upsilon}) - \mathcal{R}_{L, P}^*
\leq 4 \inf_{f \in \mathcal{F}} \left( \mathcal{R}_{L,P}(f) - \mathcal{R}_{L,P}^* \right)+ 4 \left( \frac{ c_{\mathsmaller{V}} (\tau + \ln |\mathcal{F}|)}{\neff} \right)^{1 / (2 - \vartheta)} + \frac{32 c_{\mathsmaller{B}} \tau}{\neff}
\end{align*}
hold with probability $\mu$ not less than $1 - 8 C e^{-\tau}$. 
\end{example}

Recalling that for the i.i.d.~case we have $\neff = n$, therefore, in Example \ref{ERM_example} the oracle inequality \eqref{oracleinequalityy} is thus an exact analogue to  standard oracle inequality for ERM learning from i.i.d.~processes (see e.g.~\cite[Theorem 7.2]{StCh08a}), albeit with different constants.

For further examples let us begin by briefly recalling  SVMs \cite{StCh08a}. To this end, let $X$ be a measurable space, $Y := [-1, 1]$ and $k$ be a measurable (reproducing) kernel on $X$ with reproducing kernel Hilbert space (RKHS) $H$. Given a regularization parameter $\lambda > 0$ and a convex loss $L$, SVMs find the unique solution
\begin{align*}
f_{D_n, \lambda} = \argmin_{f \in H} \left( \lambda \| f \|_H^2 + \mathcal{R}_{L,D_n} (f) \right).
\end{align*}
In particular, SVMs using the least squares loss \eqref{lsloss} are called least squares SVMs (LS-SVMs) \cite{SuVaDeDeVa02} where a primal-dual characterization is given, and also termed as kernel ridge regression in the case of zero bias term as studied in \cite{StCh08a}. SVMs using the $\tau$-pinball loss \eqref{pbloss} are called SVMs for quantile regression. To describe the approximation properties of $H$, we further need the approximation error function
\begin{align}\label{alambda}
A(\lambda) := \inf_{f \in H} \left( \lambda \| f \|_H^2 + \mathcal{R}_{L,P}(f) - \mathcal{R}_{L, P}^* \right), \,\,\,\,\,\,\,\,\lambda > 0,
\end{align}
and denote $f_{P,\lambda}$ as the population version of $f_{D_n, \lambda}$, which is given by 
\begin{align}\label{f_0}
f_{P,\lambda} := \argmin_{f \in H} \left( \lambda \| f \|_H^2 + \mathcal{R}_{L,P}(f) - \mathcal{R}_{L, P}^* \right).
\end{align}
The next example discusses learning rates for LS-SVMs using a given generic kernel.

\begin{example}[Generic Kernels]   \label{LSSVMfixedkernels} 
Let $(X,\mathcal{X})$ be a measurable space, $Y = [-1, 1]$, and $\mathcal{Z}$ be a stochastic process satisfying Assumption \ref{BernsteinInequalityGeneral}. Furthermore, let $L$ be the least squares loss and $H$ be an RKHS over $X$ such that the closed unit ball $B_H$ of $H$ satisfies
\begin{align*}
\ln \mathcal{N} (B_H, \| \cdot \|_{\infty}, \varepsilon) 
\leq a \varepsilon^{- 2 p}, 
\,\,\,\,\,\,\,\,
\varepsilon > 0,
\end{align*}
for some constants  $p\in (0,1]$ and $a>0$. In addition, assume that the approximation error function satisfies $A(\lambda) \leq c \lambda^{\beta}$ for some $c > 0$, $\beta \in (0, 1]$, and all $\lambda > 0$. 

Recall that for SVMs we always have $f_{D_n, \lambda} \in \lambda^{-1/2} B_H$, see \cite[(5.4)]{StCh08a}. Consequently we only need to consider the hypothesis set $\mathcal{F} = \lambda^{-1/2} B_H$. Then, \eqref{Fr} implies that 
$\mathcal{F}_r \subset r^{1/2} \lambda^{-1/2} B_H$ and consequently we find
\begin{align}
\ln \mathcal{N}(\mathcal{F}_r, \|\cdot\|_{\infty}, \varepsilon) \leq a \lambda^{-p} \varepsilon^{-2p} r^p.
\end{align}
Thus, we can define the function $\varphi$ in \eqref{coveringnumber} as $\varphi(\varepsilon) := a \lambda^{-p} \varepsilon^{- 2 p}$. 
For the least squares loss, the variance bound \eqref{variancebound} is valid with $\vartheta = 1$, hence
the condition \eqref{minradius} is satisfied if
\begin{align}\label{minradius2}
r \geq \max \left\{ 
\left( c_{\mathsmaller{V}} 2^{1+3p} a \right)^{\frac{1}{1-p}} \lambda^{-\frac{p}{1-p}} \neff^{- \frac{1}{1-p}}       \varepsilon^{-\frac{2p}{1-p}},\frac{2 c_{\mathsmaller{V}} \tau}{\neff},\frac{8 c_{\mathsmaller{B}} B_0 \tau}{\neff},r^* \right\}.
\end{align}
Therefore, let $r$ be the sum of the terms on the right-hand side. Since for large $n$ the first and next-to-last term in \eqref{minradius2} dominate, the oracle inequality \eqref{oracleinequalityy} becomes
\begin{align*}
\lambda \|f_{D_n,\lambda}\|_H^2 + \mathcal{R}_{L,P} (\wideparen{f}_{D_n, \lambda}) - \mathcal{R}_{L,P}^*
& \leq 4 \lambda \|f_{P,\lambda}\|_H^2 + 4 \mathcal{R}_{L, P}(f_{P,\lambda}) - 4 \mathcal{R}_{L, P}^* + 4 r + 5 \varepsilon  
\\
& \leq C \Bigl( \lambda^{\beta} + \lambda^{-\frac{p}{1-p}} \neff^{- \frac{1}{1-p}} \varepsilon^{-\frac{2p}{1-p}}+ \lambda^{\beta-1} \neff^{-1} \tau + \varepsilon \Bigr) \, ,
\end{align*}
where $f_{P,\lambda}$ is defined in \eqref{f_0} and $C$ is a constant independent of $n$, $\lambda$, $\tau$, or $\varepsilon$. Now optimizing over $\varepsilon$, we then see by \cite[Lemma A.1.7]{StCh08a}  that the LS-SVM using $\lambda_n := \neff^{- \rho / \beta}$ learns with the rate $\neff^{- \rho}$, where
\begin{equation}\label{rho-gen}
\rho := \min \left\{ \beta, \frac{\beta}{\beta + p \beta + p} \right\}.
\end{equation}
\end{example}

In particular, for geometrically $\alpha$-mixing processes, we obtain the learning rate $n^{- \alpha \rho}$, where $\alpha := \frac{\gamma}{\gamma+1}$ and $\rho$ as in \eqref{rho-gen}. Let us compare this rate with the ones previously established for LS-SVMs in the literature. For example, \cite{StCh09a} proved a rate of the form 
\begin{displaymath}
   n^{-\alpha \min\{\beta, \frac{\beta}{\beta+2p\beta+p}\}}
\end{displaymath}
under exactly the same assumptions. 
Since $\beta>0$ and $p>0$, our rate is always better than that of \cite{StCh09a}. In addition, \cite{Feng12a} generalized the rates of \cite{StCh09a} to regularization terms of the form $\lambda\|\cdot\|_H^q$ with $q\in (0,2]$. The resulting rates are again always slower than the ones established in this work. For the standard regularization term, that is $q=2$, \cite{XuCh08a}  established the rate
\begin{align*}
n^{- \frac{\alpha \beta}{2 p + 1}},
\end{align*}
which is  always slower than ours, too.  Finally, in the case $p=1$,  \cite{SuWu09a} established the rate
\begin{align*}
 n^{- \frac{2 \alpha \beta}{\beta + 3}},
\end{align*}
which was subsequently improved to 
\begin{align*}
n^{- \frac{3 \alpha \beta}{2\beta + 4}}
\end{align*}
in \cite{SuWu10a}.
The latter rate is worse than ours, if and only if $(1+\beta)(1+3p)\leq 5$. In particular, for $p\in (0,1/2]$ we always get better rates.
Furthermore, the analysis of  \cite{SuWu09a,SuWu10a} is restricted to LS-SVMs, while our results hold for rather generic learning algorithms.

\begin{example}[Smooth Kernels]\label{ex-smooth-kernesl}
Let $X \subset \mathbb{R}^d$ be a compact subset, $Y = [-1, 1]$, and $\mathcal{Z}$ be a stochastic process satisfying Assumption \ref{BernsteinInequalityGeneral}. Furthermore, let $L$ be the least squares loss and  $H = W^m(X)$ be a Sobolev space with smoothness $m > d/2$. Then it is well-known, see e.g. \cite{StHuSc09b} or \cite[Theorem 6.26]{StCh08a}, that
\begin{align*}
\ln \mathcal{N} (B_H, \|\cdot\|_{\infty}, \varepsilon) \leq a \varepsilon^{- 2 p},
\,\,\,\,\,\,\,\,
\varepsilon > 0,
\end{align*}
where $p := \frac{d}{2m}$ and $a > 0$ is some constant. Let us additionally assume that  the marginal distribution $P_X$ is absolutely 
continuous with respect to the uniform distribution, where the corresponding density is bounded away from 0 and $\infty$. Then there exists a constant $C_p > 0$ such that
\begin{align*}
\|f\|_{\infty} \leq C_p \|f\|_H^p \|f\|_{L_2(P_X)}^{1-p}, 
\,\,\,\,\,\,\,\,
f \in H
\end{align*}
for the same $p$, see \cite{MeNe10a} and \cite[Corollary 3]{StHuSc09b}. Consequently, we can bound $B_0 \leq \lambda^{(\beta - 1)p}$ as in \cite{StHuSc09b}. Moreover, the assumption on the approximation error function is satisfied for $\beta := s/m$, whenever $f_{L,P}^* \in W^s(X)$ and $s \in (0, m]$. Therefore, the resulting learning rate is
\begin{align}\label{genratesmooth}
\neff^{- \frac{2 s}{2 s + d + d s / m}}\, .
\end{align}
\end{example}

Note that in the i.i.d.~case, where $\neff = n$, this rate is worse than the optimal rate $n^{- \frac{2 s}{2 s + d}}$, where the discrepancy is the term $ds/m$ in the denominator. However, this difference can be made arbitrarily small by picking a sufficiently large $m$, 
that is, a sufficiently smooth kernel $k$. Moreover, in this case, for geometrically $\alpha$-mixing processes, the rate \eqref{genratesmooth} becomes
\begin{align*}
n^{- \frac{2 s \alpha}{2 s + d + d s / m}}\, ,
\end{align*}
where $\alpha := \frac{\gamma}{\gamma+1}$. 
Comparing this rate with the one from \cite{SuWu10a}, it turns out that their rate is worse than ours, if $m \geq \frac{1}{16}(2s+3d+\sqrt{4s^2+108sd+9d^2})$. Note that by the constraint $s\leq m$, the latter is always satisfied for $m\geq d$.

In the following, we are mainly interested in the commonly used Gaussian kernels $k_{\sigma} : X \times X \to \mathbb{R}$ defined by
\begin{align*}
k_{\sigma} (x, x') := \exp \bigl( - \|x - x'\|_2^2/\sigma^2 \bigr), 
\,\,\,\,\,\,\,\,
x, x' \in X,
\end{align*}
where $X \subset \mathbb{R}^d$ is a nonempty subset and $\sigma > 0$ is a free parameter called the width. We write $H_{\sigma}$ for the corresponding RKHSs, which are described in some detail in \cite{StHuSc06a}. The entropy numbers for Gaussian kernels \cite[Theorem 6.27]{StCh08a} and the equivalence of covering and entropy numbers \cite[Lemma 6.21]{StCh08a} yield that
\begin{align}\label{covernumber3}
\ln \mathcal{N}(B_{H_\sigma},\|\cdot\|_{\infty},\varepsilon) \leq a \sigma^{-d} \varepsilon^{- 2 p},\,\,\,\,\,\,\,\,\varepsilon > 0,
\end{align}
for some constants $a > 0$ and $p \in (0, 1)$. Then \eqref{Fr} implies $\mathcal{F}_r \subset r^{1/2} \lambda^{-1/2} B_{H_\sigma}$
and consequently
\begin{align*}
\ln \mathcal{N}(\mathcal{F}_r,\|\cdot\|_{\infty},\varepsilon) \leq a \sigma^{-d} \lambda^{-p} \varepsilon^{-2p} r^p.
\end{align*}
Therefore, we can define the function $\varphi$ in Theorem \ref{oracleinequality} as
\begin{align}\label{varphi}
\varphi(\varepsilon) := a \sigma^{-d} \lambda^{-p} \varepsilon^{-2p}.
\end{align}
Moreover,  \cite[Section 2]{EbSt11a} shows that there exists a constant $c>0$ such that for all $\lambda>0$ and all $\sigma\in (0,1]$, there is an $f_0\in H_\sigma$ with $\|f_0\|_{\infty} \leq c$ and 
\begin{displaymath}
A(\lambda) \leq \lambda \|f_0\|_{H_\sigma}^2 +  \mathcal{R}_{L,P}(f_0) -  \mathcal{R}_{L,P}^*  \leq c \lambda \sigma^{-d} + c \sigma^{2t}\, .
\end{displaymath}

\begin{example}[Gaussian Kernels]\label{LSSVMgaussiankernels}
Let $Y := [- M, M]$ for some $M > 0$, $Z := \mathbb{R}^d \times Y$, $\mathcal{Z}$ be a stochastic process satisfying Assumption \ref{BernsteinInequalityGeneral}, and $P$ be a distribution on $Z$ whose marginal distribution on $\mathbb{R}^d$ is concentrated on $X \subset B_{\ell_2^d}$ and absolutely continuous w.r.t. the Lebesgue measure $\mu$ on $\mathbb{R}^d$. We denote the corresponding density $g : \mathbb{R}^d \to [0, \infty)$ and assume $\mu(\partial X) = 0$ and $g \in L_{q}(\mu)$ for some $q \geq 1$. Moreover, assume that the Bayes decision function $f^*_{L, P} = \mathbb{E}_P(Y | x)$ satisfies $f^*_{L, P} \in L_2(\mu) \cap L_{\infty}(\mu)$ as well as $f^*_{L, P} \in B_{2s, \infty}^t$ for some $t \geq 1$ and $s \geq 1$ with $\frac{1}{q} + \frac{1}{s} = 1$. Here, $B_{2s,\infty}^t$ denotes the Besov space with the smoothness parameter $t$, see also \cite[Section 2]{EbSt11a}. Recall that, for the least squares loss, the variance bound \eqref{variancebound} is valid with $\vartheta = 1$. 
Consequently, Condition \eqref{minradius} is satisfied if
\begin{align}\label{minradius3}
r \geq \max \left\{ 
\left( c_{\mathsmaller{V}} 2^{1+3p} a \right)^{\frac{1}{1-p}} \sigma^{- \frac{d}{1-p}} \lambda^{-\frac{p}{1-p}} \neff^{- \frac{1}{1-p}} \varepsilon^{-\frac{2p}{1-p}},\frac{2 c_{\mathsmaller{V}} \tau}{\neff},\frac{8 c_{\mathsmaller{B}} B_0 \tau}{\neff},r^* \right\}.
\end{align}
Note that in the right-hand side of \eqref{minradius3}, the first term dominates when $n$ goes large. In this context, the oracle inequality \eqref{oracleinequalityy}  becomes
\begin{align*}
\lambda \|f_{D_n,\lambda}\|_{H_\sigma}^2 + \mathcal{R}_{L,P}(\wideparen{f}_{D_n, \lambda}) - \mathcal{R}_{L, P}^*
 \leq C \Bigl( \lambda \sigma^{-d} +  \sigma^{2t} + \sigma^{- \frac{d}{1-p}} \lambda^{-\frac{p}{1-p}} \neff^{- \frac{1}{1-p}}           \varepsilon^{-\frac{2p}{1-p}} \tau + \varepsilon \Bigr).
\end{align*}
Here $C$ is a constant independent of $n$, $\lambda$, $\sigma$, $\tau$, or $\varepsilon$. Again, optimizing over $\varepsilon$ together with some standard techniques, see \cite[Lemma A.1.7]{StCh08a}, we then see that for all $\xi > 0$, the LS-SVM using Gaussian RKHS $H_{\sigma}$ and
\begin{align*}
\lambda_n = n_{\text{eff}}^{- 1} ~~~~ \textrm{and} ~~~~ \sigma_n = n_{\text{eff}}^{- \frac{1}{2 t + d}} \ ,    
\end{align*}
learns with the rate 
\begin{align}\label{lrlssvmgaussianGF}
n_{\text{eff}}^{- \frac{2 t}{2 t + d} + \xi}. 
\end{align}
\end{example}

In the i.i.d.~case we have $n_{\text{eff}} = n$, and hence the learning rate \eqref{lrlssvmgaussianGF} becomes
\begin{align}\label{iidGF}
n^{- \frac{2 t}{2 t + d} + \xi} \, .
\end{align}
Recall  that  modulo the arbitrarily small $\xi>0$ these learning rates are essentially optimal, see e.g. \cite[Theorem 13]{StHuSc09b} or \cite[Theorem 3.2]{GyKoKrWa02}. Moreover, for geometrically $\alpha$-mixing processes, the rate \eqref{lrlssvmgaussianGF} becomes
\begin{align*}
n^{- \frac{2 t}{2 t + d} \alpha + \xi} \, ,
\end{align*}
where $\alpha := \frac{\gamma}{\gamma+1}$. This rate is optimal up to the factor $\alpha$ and the additional $\xi$ in the exponent. Particularly, for restricted geometrically $\alpha$-mixing processes, geometrically $\alpha$-mixing Markov chains, $\phi$-mixing processes, we obtain the essentially optimal learning rates \eqref{iidGF}. Moreover, the same essentially optimal learning rates can be achieved for (time-reversed) geometrically $\mathcal{C}$-mixing processes, if we additionally assume $f_{L,P}^* \in \mathrm{Lip}(\mathbb{R}^d)$, see also \cite[Example 4.7]{HaSt15a}.

In the last example, we will briefly discuss learning rates for SVMs for quantile regression. For more information on such SVMs we refer to  \cite[Section 4]{EbSt11a}.

\begin{example}[Quantile Regression with Gaussian Kernels]\label{quantile_regression_Gaussian} 
Let $Y := [- 1, 1]$, $Z := \mathbb{R}^d \times Y$, $\mathcal{Z}$ be a stochastic process satisfying Assumption \ref{BernsteinInequalityGeneral}, $P$ be a distribution on $Z$, and $Q$ be the marginal distribution of $P$ on $\mathbb{R}^d$.
Assume that $X := \mathrm{supp} \, Q \subset B_{\ell_2^d}$ and that for $Q$-almost all $x \in X$, the conditional probability $P(\cdot|x)$ is absolutely continuous w.r.t.~the Lebesgue measure on $Y$ and the conditional densities $h(\cdot, x)$ of $P(\cdot|x)$ are uniformly bounded away from $0$ and $\infty$, see also \cite[Example 4.5]{EbSt11a}. Moreover, assume that $Q$ is absolutely continuous w.r.t.~the Lebesgue measure on $X$ with associated density $g \in L_{u}(X)$ for some $u \geq 1$. For $\tau \in (0, 1)$, let $f^*_{\tau, P} : \mathbb{R}^d \to \mathbb{R}$ be a conditional $\tau$-quantile function that satisfies $f^*_{\tau, P} \in L_2(\mu) \cap L_{\infty}(\mu)$.
In addition, we assume that  $f^*_{\tau, P} \in B_{2s, \infty}^t$ for some $t \geq 1$ and $s \geq 1$ such that $\frac{1}{s} + \frac{1}{u} = 1$. Then \cite[Theorem 2.8]{StCh11a} yields a variance bound of the form 
\begin{align}\label{Bernstein_quantile}
\mathbb{E}_{P} (L_{\tau} \circ \wideparen{f} - L_{\tau} \circ f^*_{\tau,P})^2 \leq V\cdot \mathbb{E}_P (L_{\tau} \circ \wideparen{f} -L_{\tau} \circ f^*_{\tau,P}) \, ,
\end{align}
for all $f:X\to \mathbb{R}$, 
where $V$ is a suitable constant and $L_\tau$ is the $\tau$-pinball loss. Then, following similar arguments with those in Example \ref{LSSVMgaussiankernels}, with the same choices of $\lambda_n$ and $\sigma_n$, the same rates can be obtained as in Example \ref{LSSVMgaussiankernels}. 
\end{example}

Here, we give two remarks on Example \ref{quantile_regression_Gaussian}. First, it is noted that the Bernstein condition \eqref{Bernstein_quantile} holds when the distribution $P$ is of a $\tau$-quantile of $p$-average type $q$ in the sense of Definition $2.6$ in \cite{StCh11a}. Two distributions of this type can be found in Examples $2.3$ and $2.4$ in \cite{christmann2007svms}. On the other hand, the rates obtained in Example  \ref{quantile_regression_Gaussian} are in fact for the excess $L_\tau$-risk. However, since \cite[Theorem 2.7]{StCh11a} shows 
\begin{displaymath}
\|\wideparen{f} - f_{\tau,P}^*\|_{L_2(P_X)}^2 
\leq c \bigl( \mathcal{R}_{L_\tau,P} (\wideparen{f}\,) - \mathcal{R}_{L_\tau,P}^* \bigr)
\end{displaymath}
for some constant $c > 0$ and all $f : X \to \mathbb{R}$, we also obtain the same rates for $\|\wideparen f - f_{\tau,P}^*\|_{L_2(P_X)}^2$. Last but not least, optimality for various mixing processes can be discussed along the lines of LS-SVMs.

\section{Conclusions}\label{Conclusion}
In the present paper, we proposed a unified learning theory approach to studying learning schemes sampling from various commonly 
investigated stationary mixing processes that include geometrically $\alpha$-mixing processes, geometrically $\alpha$-mixing Markov chains, $\phi$-mixing processes, and geometrically $\mathcal{C}$-mixing processes. The proposed approach is considered to be unified in the following sense: First, in our study, the empirical processes of the above-mentioned mixing processes were assumed to satisfy a generalized Bernstein-type inequality, which includes many commonly considered cases; Second, by instantiating the generalized Bernstein-type inequality to different scenarios, we illustrated the effective number of observations for different mixing processes; Third, based on the above generalized Bernstein-type concentration assumption, a generalized sharp oracle inequality was established within the statistical learning theory framework. Finally, faster or at least comparable learning rates can be obtained by applying the established oracle inequality to various learning schemes with different mixing processes.

\section*{Acknowledgement} 
\small{The authors would like to thank the editor and the reviewers for their insightful comments and helpful suggestions that improved the quality of this paper. The research leading to these results has received funding from the European Research Council under the European Union's Seventh Framework Programme (FP7/2007-2013) / ERC AdG A-DATADRIVE-B (290923). This paper reflects only the authors' views, the Union is not liable for any use that may be made of the contained information;  Research Council KUL: GOA/10/09 MaNet, CoE PFV/10/002 (OPTEC),  BIL12/11T; PhD/Postdoc grants; Flemish Government: FWO: PhD/Postdoc grants, projects: G.0377.12 (Structured systems), G.088114N (Tensor based data similarity); IWT: PhD/Postdoc grants, projects: SBO POM (100031); iMinds Medical Information Technologies SBO 2014; Belgian Federal Science Policy Office: IUAP P7/19 (DYSCO, Dynamical systems, control and optimization, 2012-2017). }

\appendix
\section*{Appendix}

\subsection*{Proof of Theorem \ref{oracleinequality} in Section \ref{MainResults}}
Since the proof of Theorem \ref{oracleinequality} is rather complicated, we first describe its main steps briefly: First we decompose the regularized excess risk into an approximation error term and two stochastic error terms. The approximation error and the first stochastic error term can be estimated by standard   techniques. Similarly, the first step in the estimation of the second error term is a rather standard quotient approach, see e.g.~\cite[Theorem 7.20]{StCh08a}, which allows for localization with respect to both the variance and the regularization. Due to the absence of tools from empirical process theory, however, the remaining estimation steps become more involved. To be more precise, we split the ``unit ball'' of the hypothesis space $\mathcal{F}$ into disjoint ``spheres''. For each sphere, we then use localized covering numbers and the generalized Bernstein-type inequality from Assumption \ref{BernsteinInequalityGeneral}, and  the resulting estimates   are then  combined using the peeling method. This yields a quasi-geometric series with  rate smaller than 1 if the radius of the innermost ball is sufficiently large. As a result, the estimated error probability on the whole ``unit ball'' nearly equals the estimated error probability of the innermost ``ball'', which unsurprisingly leads to  a significant improvement compared to \cite{StCh09a}. 
 
Before we prove Theorem \ref{oracleinequality}, we need to reformulate \eqref{BersteinGeneral}. Setting $\tau := \frac{\varepsilon^2 \neff}{c_{\sigma} \sigma^2 + \varepsilon c_{\mathsmaller{B}} B}$, with some simple transformations we obtain
\begin{align}\label{BersteinGeneral2}
\mu \left( \left\{ \omega \in \Omega : \frac{1}{n} \sum_{i=1}^n h(Z_i(\omega)) \geq \sqrt{\frac{\tau c_{\sigma} \sigma^2}{\neff}} + 
\frac{c_{\mathsmaller{B}} B \tau}{\neff} \right\} \right) \leq C e^{- \tau}   
\end{align}
for all $\tau > 0$ and $n \geq n_0$.

\begin{proof}[Proof of Theorem \ref{oracleinequality}]
\textbf{Main Decomposition.}
For $f : X \to \mathbb{R}$ we define $h_f := L \circ f - L \circ f_{L,P}^*$. 
By the definition of $f_{D_n,\Upsilon}$, we then have 
\begin{align*}
\Upsilon(f_{D_n,\Upsilon}) + \mathbb{E}_{D_n} h_{\wideparen{f}_{D_n, \Upsilon}} \leq \Upsilon (f_0) + \mathbb{E}_{D_n} h_{f_0} + \delta,
\end{align*} 
and consequently we obtain
\begin{align}\label{splitting}
\begin{split}
\Upsilon(f_{D_n,\Upsilon}) &  + \mathcal{R}_{L,P} (\wideparen{f}_{D_n,\Upsilon}) - \mathcal{R}_{L,P}^*= \Upsilon(f_{D_n,\Upsilon}) + \mathbb{E}_P h_{\wideparen{f}_{D_n,\Upsilon}}
\\
& \leq \Upsilon(f_0) + \mathbb{E}_{D_n} h_{f_0} - \mathbb{E}_{D_n} h_{\wideparen{f}_{D_n, \Upsilon}} 
       + \mathbb{E}_P h_{\wideparen{f}_{D_n,\Upsilon}} + \delta
\\
& = (\Upsilon(f_0) + \mathbb{E}_P h_{f_0}) + (\mathbb{E}_{D_n} h_{f_0} - \mathbb{E}_P h_{f_0}) + (\mathbb{E}_P h_{\wideparen{f}_{D_n,\Upsilon}} - \mathbb{E}_{D_n} h_{\wideparen{f}_{D_n, \Upsilon}}) + \delta.    
\end{split}
\end{align}

\textbf{Estimating the First Stochastic  Term.} 
Let us first bound the term $\mathbb{E}_{D_n} h_{f_0} - \mathbb{E}_P h_{f_0}$. To this end, we further split this difference into
\begin{align}\label{splitting2}
\mathbb{E}_{D_n} h_{f_0} - \mathbb{E}_P h_{f_0} = \left( \mathbb{E}_{D_n} ( h_{f_0} - h_{\wideparen{f}_0} ) - \mathbb{E}_P ( h_{f_0} - h_{\wideparen{f}_0} ) \right)  + ( \mathbb{E}_{D_n} h_{\wideparen{f}_0} - \mathbb{E}_P h_{\wideparen{f}_0} ). 
\end{align}
Now $L \circ f_0 - L \circ \wideparen{f}_0 \geq 0$ implies $h_{f_0} - h_{\wideparen{f}_0} = L \circ f_0 - L \circ \wideparen{f}_0 \in [0, B_0]$, and hence we obtain
\begin{align*}
\mathbb{E}_P \left( ( h_{f_0} - h_{\wideparen{f}_0} ) - \mathbb{E}_P ( h_{f_0} - h_{\wideparen{f}_0} ) \right)^2 \leq \mathbb{E}_P ( h_{f_0} - h_{\wideparen{f}_0} )^2 \leq B_0 \mathbb{E}_P ( h_{f_0} - h_{\wideparen{f}_0} ).
\end{align*}
Inequality \eqref{BersteinGeneral2} applied to $h := ( h_{f_0} - h_{\wideparen{f}_0} ) - \mathbb{E}_P ( h_{f_0} - h_{\wideparen{f}_0} )$ thus shows that
\begin{align*}
\mathbb{E}_{D_n} ( h_{f_0} - h_{\wideparen{f}_0} ) - \mathbb{E}_P ( h_{f_0} - h_{\wideparen{f}_0} ) \leq \sqrt{\frac{\tau c_{\sigma} B_0 \mathbb{E}_P ( h_{f_0} - h_{\wideparen{f}_0} )}{\neff}} + \frac{c_{\mathsmaller{B}} B_0 \tau}{\neff}
\end{align*}
holds with probability $\mu$ not less than $1 - C e^{-\tau}$. Moreover, using $\sqrt{a b} \leq \frac{a}{2} + \frac{b}{2}$, we find
\begin{align*}
\sqrt{\neff^{-1} \tau c_{\sigma} B_0 \mathbb{E}_P ( h_{f_0} - h_{\wideparen{f}_0} )} \leq \mathbb{E}_P ( h_{f_0} - h_{\wideparen{f}_0} ) + \neff^{-1} c_{\sigma} B_0 \tau / 4,
\end{align*}
and consequently we have with probability $\mu$ not less than $1 - C e^{-\tau}$ that
\begin{align}\label{splitting21estimate}
\mathbb{E}_{D_n} ( h_{f_0} - h_{\wideparen{f}_0} ) - \mathbb{E}_P ( h_{f_0} - h_{\wideparen{f}_0} ) \leq \mathbb{E}_P ( h_{f_0} - h_{\wideparen{f}_0} ) + \frac{7 c_{\mathsmaller{B}} B_0 \tau}{4 \neff}. 
\end{align}
In order to bound the remaining term in \eqref{splitting2}, that is $\mathbb{E}_{D_n} h_{\wideparen{f}_0} - \mathbb{E}_P h_{\wideparen{f}_0}$, we first observe that \eqref{lipschitz} implies $\|h_{\wideparen{f}_0}\|_{\infty} \leq 1$, and hence we have $\|h_{\wideparen{f}_0} - \mathbb{E}_P h_{\wideparen{f}_0}\|_{\infty} \leq 2$. Moreover, \eqref{variancebound} yields
\begin{align*}
\mathbb{E}_P ( h_{\wideparen{f}_0} - \mathbb{E}_P h_{\wideparen{f}_0} )^2 \leq \mathbb{E}_P h_{\wideparen{f}_0}^2 \leq V (\mathbb{E}_P h_{\wideparen{f}_0})^{\vartheta}.
\end{align*}
In addition, if $\vartheta \in (0, 1]$, the first inequality in \cite[Lemma 7.1]{StCh08a} implies for $q := \frac{2}{2 - \vartheta}$, $q' := \frac{2}{\vartheta}$, $a := (\neff^{-1} c_{\sigma} 2^{- \vartheta} \vartheta^{\vartheta} V \tau)^{1/2}$, and $b := (2 \vartheta^{-1} \mathbb{E}_P h_{\wideparen{f}_0})^{\vartheta/2}$, that
\begin{align*}
\sqrt{\frac{c_{\sigma} V \tau (\mathbb{E}_P h_{\wideparen{f}_0})^{\vartheta}}{\neff}} \leq \left( 1 - \frac{\vartheta}{2} \right) 
     \left( \frac{c_{\sigma} 2^{- \vartheta} \vartheta^{\vartheta} V \tau}{\neff} \right)^{\frac{1}{2 - \vartheta}} + \mathbb{E}_P h_{\wideparen{f}_0} 
\leq \left( \frac{c_{\sigma} V \tau}{\neff} \right)^{\frac{1}{2 - \vartheta}} + \mathbb{E}_P h_{\wideparen{f}_0}.
\end{align*}
Since $\mathbb{E}_P h_{\wideparen{f}_0} \geq 0$, this inequality also holds for $\vartheta = 0$, and hence \eqref{BersteinGeneral2} shows that 
we have
\begin{align*}
\mathbb{E}_{D_n} h_{\wideparen{f}_0} - \mathbb{E}_P h_{\wideparen{f}_0} \leq \mathbb{E}_P h_{\wideparen{f}_0} 
     + \left( \frac{c_{\sigma} V \tau}{\neff} \right)^{\frac{1}{2 - \vartheta}} + \frac{2 c_{\mathsmaller{B}} \tau}{\neff} 
\end{align*}
with probability $\mu$ not less than $1 - C e^{-\tau}$. By combining this estimate with \eqref{splitting21estimate} and \eqref{splitting2}, we now obtain that with probability $\mu$ not less than $1 - 2 C e^{-\tau}$ we have
\begin{align}\label{splitting2estimate}
\mathbb{E}_{D_n} h_{f_0} - \mathbb{E}_P h_{f_0} \leq \mathbb{E}_P h_{f_0} + \left( \frac{c_{\sigma} V \tau}{\neff} \right)^{\frac{1}{2 - \vartheta}} + \frac{2 c_{\mathsmaller{B}} \tau}{\neff} + \frac{7 c_{\mathsmaller{B}} B_0 \tau}{4 \neff}, 
\end{align}
since $1 \leq B_0$, i.e., we have established a bound on the second term in \eqref{splitting}. 

\textbf{Estimating the Second Stochastic  Term.}
For the third term in \eqref{splitting} let us first consider the case $\neff < c_{\mathsmaller{V}} (\tau + \varphi(\varepsilon/2) 2^p r^p)$. Combining \eqref{splitting2estimate} with \eqref{splitting} and using $1 \leq B_0$, $1 \leq V$, $c_{\sigma} V \leq c_{\mathsmaller{V}}$, $2 \leq 4^{1/(2-\vartheta)}$, and $\mathbb{E}_P h_{\wideparen{f}_{D_n, \Upsilon}} - \mathbb{E}_{D_n} h_{\wideparen{f}_{D_n, \Upsilon}} \leq 2$, then we find
\begin{align*}
\Upsilon(f_{D_n,\Upsilon}) + \mathcal{R}_{L,P} (\wideparen{f}_{D_n,\Upsilon}) - \mathcal{R}_{L,P}^*
& \leq \Upsilon(f_0) + 2 \mathbb{E}_P h_{f_0} 
       + \left( \frac{c_{\sigma} V \tau}{\neff} \right)^{\frac{1}{2 - \vartheta}} 
       + \frac{2 c_{\mathsmaller{B}} \tau}{\neff} + \frac{7 c_{\mathsmaller{B}} B_0 \tau}{4 \neff}
\\
& \phantom{=} 
       + (\mathbb{E}_P h_{\wideparen{f}_{D_n, \Upsilon}} - \mathbb{E}_{D_n} h_{\wideparen{f}_{D_n, \Upsilon}} ) 
       + \delta
\\
& \leq \Upsilon(f_0) + 2 \mathbb{E}_P h_{f_0} 
       + \left( \frac{c_{\sigma} V (\tau + \varphi(\varepsilon/2) 2^p r^p)}{\neff} 
                \right)^{\frac{1}{2 - \vartheta}} 
       + \frac{4 c_{\mathsmaller{B}} B_0 \tau}{\neff}
\\
& \phantom{=} 
       + 2 \left( \frac{c_{\mathsmaller{V}} (\tau + \varphi(\varepsilon/2) 2^p r^p)}{\neff} 
                  \right)^{\frac{1}{2 - \vartheta}} 
       + \delta
\\
& \leq 2 \Upsilon(f_0) + 4 \mathbb{E}_P h_{f_0} 
       + 3 \left( \frac{c_{\mathsmaller{V}} (\tau + \varphi(\varepsilon/2) 2^p r^p)}{\neff} \right)^{\frac{1}{2 - \vartheta}} 
       + \frac{8 c_{\mathsmaller{B}} B_0 \tau}{\neff} 
       + 2 \delta
\end{align*}
with probability $\mu$ not less than $1 - 2 C e^{-\tau}$. It thus remains to consider the case $\neff \geq c_{\mathsmaller{V}} (\tau + \varphi(\varepsilon/2) 2^p r^p)$. 

\textbf{Introduction of the Quotients.}
To establish a non-trivial bound on the term $\mathbb{E}_P h_{\wideparen{f}_D} - \mathbb{E}_{D_n} h_{\wideparen{f}_D}$ in \eqref{splitting}, we define functions
\begin{align*}
g_{f,r} := \frac{\mathbb{E}_P h_{\wideparen{f}} - h_{\wideparen{f}}}{\Upsilon(f) + \mathbb{E}_P h_{\wideparen{f}} + r}, 
\,\,\,\,\,\,\,\,f \in \mathcal{F}, \,\,r > r^*.
\end{align*}
For $f \in \mathcal{F}$, we have $\| \mathbb{E}_P h_{\wideparen{f}} - h_{\wideparen{f}} \|_{\infty} \leq 2$. Moreover, for $f \in \mathcal{F}_r$, the variance bound \eqref{variancebound} implies
\begin{align}\label{variancebound2} 
\mathbb{E}_P ( h_{\wideparen{f}} - \mathbb{E}_P h_{\wideparen{f}} )^2 \leq \mathbb{E}_P h_{\wideparen{f}}^2 \leq V (\mathbb{E}_P h_{\wideparen{f}})^{\vartheta} \leq V r^{\vartheta}. 
\end{align}

\textbf{Peeling.}
For a fixed $r \in (r^*, 1]$, let $K$ be the largest integer satisfying $2^K r \leq 1$. Then we can get the following disjoint partition of the function set $\mathcal{F}_1$:
\begin{align*}
\mathcal{F}_1 
\subset \mathcal{F}_r \cup \bigcup_{k=1}^{K+1} \left( \mathcal{F}_{2^k r} \backslash \mathcal{F}_{2^{k-1} r} \right).
\end{align*}
We further write $\overline{C}_{\varepsilon, r, 0}$ for a minimal $\varepsilon$-net of $\mathcal{F}_r$ and $\overline{C}_{\varepsilon, r, k}$ for minimal $\varepsilon$-nets of $\mathcal{F}_{2^k r} \backslash \mathcal{F}_{2^{k-1} r}$, $1 \leq k \leq K+1$, respectively. 
Then the union of these nets $\bigcup_{k=0}^{K+1} \overline{C}_{\varepsilon, r, k} =: \overline{C}_{\varepsilon, 1}$ is an $\varepsilon$-net of the set $\mathcal{F}_1$. Moreover, we define 
\begin{align*}
\widetilde{\mathcal{C}}_{\varepsilon, r, k} := \bigcup_{l=0}^k \overline{C}_{\varepsilon, r, l}, \,\,\,\,\,\,\,\,0 \leq k \leq K + 1,
\end{align*}
which are $\varepsilon$-nets of $\mathcal{F}_{2^k r}$ with $\widetilde{\mathcal{C}}_{\varepsilon, r, k} \subset \widetilde{\mathcal{C}}_{\varepsilon, r, k+1}$ for all $0 \leq k \leq K$, and the net $\widetilde{\mathcal{C}}_{\varepsilon, r, K+1}$
coincide with $\overline{C}_{\varepsilon, 1}$. For $A \subset B$ an elementary calculation shows that
\begin{align}\label{coveringsubset}
\mathcal{N} (A, \|\cdot\|_{\infty}, \varepsilon) \leq \mathcal{N} (B, \|\cdot\|_{\infty}, \varepsilon/2).   
\end{align}
By using \eqref{coveringsubset} for 
$\mathcal{F}_{2^k r} \backslash \mathcal{F}_{2^{k-1} r} \subset \mathcal{F}_{2^k r}$ we can estimate the cardinality of $\widetilde{\mathcal{C}}_{\varepsilon, r, k}$ by
\begin{align}\label{mcoveringnumber}
\begin{split}
|\widetilde{\mathcal{C}}_{\varepsilon, r, k}| 
& = \left| \bigcup_{l=0}^k \overline{C}_{\varepsilon, r, l} \right|
  \leq \sum_{l=0}^k |\overline{C}_{\varepsilon, r, l}|
  = \sum_{l=0}^k \mathcal{N}(\mathcal{F}_{2^k r} 
                 \backslash \mathcal{F}_{2^{k-1} r}, \|\cdot\|_{\infty}, \varepsilon)
\\
& \leq \sum_{l=0}^k \mathcal{N}(\mathcal{F}_{2^k r}, \|\cdot\|_{\infty}, \varepsilon/2) 
  \leq \sum_{l=0}^k \exp \left( \varphi(\varepsilon/2) (2^l r)^p \right)
\\
& \leq (k+1) \exp \left( \varphi(\varepsilon/2) 2^{k p} r^p \right), 
       \,\,\,\,\,\,\,\,
       0 \leq k \leq K + 1.
       \end{split}        
\end{align}
Using the peeling technique in \cite[Theorem 5.2]{HaSt14a} with $\mathcal{Z}_f := \mathbb{E}_{D_n} (\mathbb{E}_P h_{\wideparen{f}} - h_{\wideparen{f}})$, $\Gamma(f) := \Upsilon(f) + \mathbb{E}_P h_{\wideparen{f}}$, 
\begin{align*}
m_k := 
\begin{cases}
r^*       & \text{ for } k = 0, \\
2^{k-1} r & \text{ for } 1 \leq k \leq K, \\
1         & \text{ for } k = K + 1, 
\end{cases}
\end{align*}
and choosing $\epsilon = 1/4 $, we get
\begin{align*}
\mu \biggl( \sup_{f \in \overline{\mathcal{C}}_{\varepsilon, 1}} \mathbb{E}_{D_n} g_{f,r} 
            > \frac{1}{4} \biggr)
& = \mu \biggl( \sup_{f \in \overline{\mathcal{C}}_{\varepsilon, 1}} 
        \frac{\mathbb{E}_{D_n} (\mathbb{E}_P h_{\wideparen{f}} - h_{\wideparen{f}})}
             {\Upsilon(f) + \mathbb{E}_P h_{\wideparen{f}} + r} > \frac{1}{4} \biggr)
\\
& \leq \sum_{k=1}^{K+2} 
       \mu \biggl( \sup_{f \in \overline{\mathcal{C}}_{\varepsilon, r, k}} 
                   \mathbb{E}_{D_n} (\mathbb{E}_P h_{\wideparen{f}} - h_{\wideparen{f}}) 
                   > (2^{k-1} r + r)/4 \biggr)
\\
& \leq \mu \biggl( \sup_{f \in \overline{\mathcal{C}}_{\varepsilon, r, 0}} 
                   \mathbb{E}_{D_n} (\mathbb{E}_P h_{\wideparen{f}} - h_{\wideparen{f}}) 
                   > (r^* + r)/4 \biggr)
\\
& \phantom{=} 
       + \sum_{k=1}^{K+1} 
         \mu \biggl( \sup_{f \in \overline{\mathcal{C}}_{\varepsilon, r, k}} 
                     \mathbb{E}_{D_n} (\mathbb{E}_P h_{\wideparen{f}} - h_{\wideparen{f}}) 
                     > (2^{k-1}r + r)/4 \biggr)
\\
& \leq \mu \biggl( \sup_{f \in \widetilde{\mathcal{C}}_{\varepsilon, r, 1}} 
                   \mathbb{E}_{D_n} (\mathbb{E}_P h_{\wideparen{f}} - h_{\wideparen{f}}) 
                   > r/4 \biggr)
\\
& \phantom{=} 
       + \sum_{k=1}^{K+1} 
         \mu \biggl( \sup_{f \in \widetilde{\mathcal{C}}_{\varepsilon, r, k}} 
                     \mathbb{E}_{D_n} (\mathbb{E}_P h_{\wideparen{f}} - h_{\wideparen{f}}) 
                     > 2^{k-1} r/4 \biggr)
\\
& \leq 2 \sum_{k=1}^{K+1} 
         \mu \biggl( \sup_{f \in \widetilde{\mathcal{C}}_{\varepsilon, r, k}} 
                     \mathbb{E}_{D_n} (\mathbb{E}_P h_{\wideparen{f}} - h_{\wideparen{f}}) 
                     > 2^{k-3} r \biggr).
\end{align*}

\textbf{Estimating the Error Probabilities on the ``Spheres''.}
Our next goal is to estimate all the error probabilities by using \eqref{BersteinGeneral}, \eqref{variancebound2} and the union bound. 
From $\vartheta \in [0,1]$ and the estimations of the covering numbers \eqref{mcoveringnumber} follows that
\begin{align*}
& \mu \biggl( \sup_{f \in \widetilde{\mathcal{C}}_{\varepsilon, r, k}} 
              \mathbb{E}_{D_n} (\mathbb{E}_P h_{\wideparen{f}} - h_{\wideparen{f}}) 
              > 2^{k-3} r \biggr)
\\
& \leq C |\widetilde{\mathcal{C}}_{\varepsilon, r, k}| 
       \exp \biggl( - \frac{(2^{k-3} r)^2 \neff}{c_{\sigma} V (2^k r)^{\vartheta} + 2 c_{\mathsmaller{B}} (2^{k-3} r)} \biggr)
\\
& \leq C \cdot (k+1) \exp \left( \varphi(\varepsilon/2) 2^{k p} r^p \right) 
         \cdot \exp \biggl( - \frac{(2^{k-1} r)^2 \neff}
                                   {32 c_{\sigma} V (2^{k-1} r)^{\vartheta} + 8 c_{\mathsmaller{B}} (2^{k-1} r)} \biggr).
\end{align*}
For $k \geq 1$, we denote the right-hand side of this estimate by $p_k(r)$, that is
\begin{align*}
p_k(r) 
:= C \cdot (k+1) \exp \left( \varphi(\varepsilon/2) 2^{k p} r^p \right) 
     \cdot \exp \biggl( - \frac{(2^{k-1} r)^2 \neff}
                               {32 c_{\sigma} V (2^{k-1} r)^{\vartheta} + 8 c_{\mathsmaller{B}} (2^{k-1} r)} \biggr).
\end{align*}
Then we have
\begin{align*}
q_k(r) 
:= \frac{p_{k+1}(r)}{p_k(r)}
& \leq \frac{k+2}{k+1} 
       \cdot \exp \left( \varphi(\varepsilon/2) (2^{k+1} r)^p - \varphi(\varepsilon/2) (2^k r)^p \right) 
\\
& \phantom{=}
       \cdot \exp \biggl( - \frac{2^2 (2^{k-1} r)^2 \neff}
                           {32 c_{\sigma} V \cdot 2 (2^{k-1} r)^{\vartheta} + 8 c_{\mathsmaller{B}} \cdot 2 (2^{k-1} r)} 
                         + \frac{(2^{k-1} r)^2 \neff}
                           {32 c_{\sigma} V (2^{k-1} r)^{\vartheta} + 8 c_{\mathsmaller{B}} (2^{k-1} r)} \biggr)
\\
& \leq 2 \exp \left( \varphi(\varepsilon/2) 2^{k p + 1} r^p \right) 
         \cdot \exp \biggl( - \frac{(2^{k-1} r)^2 \neff}
                              {32 c_{\sigma} V (2^{k-1} r)^{\vartheta} + 8 c_{\mathsmaller{B}} (2^{k-1} r)} \biggr),
\end{align*}
and our assumption $2^k r \leq 1$, $0 \leq k \leq K$ implies
\begin{align*}
q_k(r) 
& \leq 2 \exp \left( \varphi(\varepsilon/2) 2^{kp+1} r^p \right)\cdot \exp \biggl( - \frac{(2^{k-1} r)^2 \neff}                             {32 c_{\sigma} V (2^{k-1} r)^{\vartheta} + 8 c_{\mathsmaller{B}} (2^{k-1} r)} \biggr)
\\
& \leq 2 \exp \biggl( 2^{(k-1)p} \cdot 4 r^p \varphi(\varepsilon/2) 
         - 2^{(k - 1) (2 - \vartheta)} \cdot \frac{r^{2 - \vartheta} \neff}{32 c_{\sigma} V + 8 c_{\mathsmaller{B}}} \biggr).
\end{align*}
Since $p \in (0, 1]$, $k \geq 1$ and $\vartheta, \in [0, 1]$, we have
\begin{align*}
2^{(k-1)p} \leq 2^{(k - 1) (2 - \vartheta)}.
\end{align*}
The first assumption in \eqref{minradius} implies that $r \geq \bigl( 64 (4 c_{\sigma} V + c_{\mathsmaller{B}}) \varphi(\varepsilon/2) r^p / \neff \bigr)^{1/(2-\vartheta)}$ or equivalently that 
\begin{align*}
4 r^p \varphi(\varepsilon/2) 
\leq \frac{1}{2} \cdot \frac{r^{2-\vartheta} \neff}{32 c_{\sigma} V + 8 c_{\mathsmaller{B}}},
\end{align*}
thus, using $2^{(k - 1) (2 - \vartheta)} \geq 1$, we find
\begin{align*}
q_k(r) \leq 2 \exp \biggl( - \frac{1}{2} \cdot \frac{r^{2-\vartheta} \neff}{32 c_{\sigma} V + 8 c_{\mathsmaller{B}}} \biggr). 
\end{align*}
Moreover, since $\tau \geq 1$, the first assumption in \eqref{minradius} implies also
$r \geq \bigl( 64 (4 c_{\sigma} V + c_{\mathsmaller{B}}) / \neff \bigr)^{1/(2-\vartheta)}$ or equivalently that
\begin{align*}
\frac{1}{2} \cdot \frac{r^{2-\vartheta} \neff}{32 c_{\sigma} V + 8 c_{\mathsmaller{B}}}  \geq 4,
\end{align*}
and hence $q_k(r) \leq 2 e^{-4}$, that is, $p_{k+1}(r) \leq 2 e^{-4} p_k(r)$ for all $k \geq 1$. 

\textbf{Summing all the Error Probabilities.}
From the above discussion we have
\begin{align*}
\mu \biggl( \sup_{f \in \overline{\mathcal C}_{\varepsilon, 1}} \mathbb{E}_{D_n} g_{f,r} > \frac{1}{4} \biggr) 
& \leq 2 \sum_{k=1}^{K+1} p_k(r) 
  \leq 2 \cdot p_1(r) \cdot \sum_{k=0}^K (2 e^{-4})^k 
  \leq 3 p_1(r)
\\
& = 6 \, C \exp \left( \varphi(\varepsilon/2) 2^p r^p \right) 
       \cdot \exp \biggl( - \frac{r^2 \neff}{32 c_{\sigma} V r^{\vartheta} + 8 c_{\mathsmaller{B}} r} \biggr)
\\
& \leq 6 \, C \exp \left( \varphi(\varepsilon/2) 2^p r^p \right) 
       \cdot \exp \biggl( - \frac{r^2 \neff}{32 c_{\sigma} V r^{\vartheta} + 8 c_{\mathsmaller{B}} r^{\vartheta}} \biggr)
\\
& \leq 6 \, C \exp \left( \varphi(\varepsilon/2) 2^p r^p \right) 
       \cdot  \exp \left( - \frac{r^{2-\vartheta} \neff}{32 c_{\sigma} V + 8 c_{\mathsmaller{B}}}  \right).
\end{align*}
Then once again the first assumption in \eqref{minradius} gives
\begin{align*}
r \geq \biggl( \frac{(32 c_{\sigma} V + 8 c_{\mathsmaller{B}}) (\tau + \varphi(\varepsilon/2) 2^p r^p)}{\neff} 
               \biggr)^{\frac{1}{2 - \vartheta}}
\end{align*}
and a simple transformation thus yields
\begin{align*}
\mu \left( D_n \in (X \times Y)^n : \sup_{f \in \overline{\mathcal C}_{\varepsilon, 1}} \mathbb{E}_{D_n} g_{f,r} \leq \frac{1}{4} \right) 
\geq 1 - 6 C e^{-\tau}.
\end{align*}
Consequently we see that with probability $\mu$ not less than $1 - 6 C e^{-\tau}$ we have
\begin{align}\label{splitting31}
\mathbb{E}_P h_{\wideparen{f}} - \mathbb{E}_{D_n} h_{\wideparen{f}} \leq \frac{1}{4} \left( \Upsilon(f) + \mathbb{E}_P h_{\wideparen{f}} + r \right)
\end{align}
for all $f \in \overline{\mathcal{C}}_{\varepsilon, 1}$. 
Since $r \in (0, 1]$, we have $f_{D_n, \Upsilon} \in \mathcal{F}_1$, i.e. either $f_{D_n, \Upsilon} \in \mathcal{F}_r$, or there exists an integer $k \leq K+1$ such that $f_{D_n, \Upsilon} \in \mathcal{F}_{2^k r} \backslash \mathcal{F}_{2^{k-1} r}$. Thus there exists an $f_{D_n} \in \overline{\mathcal{C}}_{\varepsilon, r, 0} \subset \mathcal{F}_r$ or $f_{D_n} \in \overline{\mathcal{C}}_{\varepsilon, r, k} \subset \mathcal{F}_{2^k r} \backslash \mathcal{F}_{2^{k-1} r}$ with $\| f_{D_n, \Upsilon} - f_{D_n} \|_{\infty} \leq \varepsilon$. By the assumed Lipschitz continuity of the clipped $L$ the latter implies
\begin{align}\label{varepsilon}
|h_{\wideparen{f}_{D_n}}(x, y) - h_{\wideparen{f}_{D_n, \Upsilon}}(x, y)| \leq |\wideparen{f}_{D_n}(x) - \wideparen{f}_{D_n, \Upsilon}(x)| 
\leq |f_{D_n}(x) - f_{D_n, \Upsilon}(x)| \leq \varepsilon      
\end{align}
for all $(x, y) \in X \times Y$. For $f_{D_n, \Upsilon}, f_{D_n} \in \mathcal{F}_r$ we obviously have
\begin{align*} 
\Upsilon(f_{D_n}) + \mathbb{E}_P h_{\wideparen{f}_{D_n}} \leq r
\end{align*}
and for the other cases $f_{D_n, \Upsilon}, f_{D_n} \in \mathcal{F}_{2^k r} \backslash \mathcal{F}_{2^{k-1} r}$ we obtain
\begin{align*} 
\Upsilon(f_{D_n}) + \mathbb{E}_P h_{\wideparen{f}_{D_n}} 
\leq 2^k r = 2 \cdot 2^{k-1} r 
\leq 2 \left( \Upsilon(f_{D_n, \Upsilon}) + \mathbb{E}_P h_{\wideparen{f}_{D_n, \Upsilon}} \right), 
\end{align*}
consequently, we always have
\begin{align}\label{Upsilon} 
\Upsilon(f_{D_n}) + \mathbb{E}_P h_{\wideparen{f}_{D_n}} \leq 2 \left( \Upsilon(f_{D_n, \Upsilon}) + \mathbb{E}_P h_{\wideparen{f}_{D_n, \Upsilon}}
\right) + r. 
\end{align}
Combining \eqref{varepsilon} with \eqref{splitting31} and \eqref{Upsilon}, we obtain
\begin{align*}
\mathbb{E}_P h_{\wideparen{f}_{D_n, \Upsilon}} - \mathbb{E}_{D_n} h_{\wideparen{f}_{D_n, \Upsilon}} \leq 
\frac{1}{2} \left( \Upsilon(f_{D_n, \Upsilon}) + \mathbb{E}_P h_{\wideparen{f}_{D_n, \Upsilon}} + \varepsilon 
+ r \right) + 2 \varepsilon
\end{align*}
with probability $\mu$ not less than $1 - 6 C e^{-\tau}$. By combining this estimate with \eqref{splitting} and \eqref{splitting2estimate}, we then obtain that
\begin{align*}
\Upsilon(f_{D_n, \Upsilon}) + \mathbb{E}_P h_{\wideparen{f}_{D_n, \Upsilon}}
& \leq \Upsilon(f_0) + 2 \mathbb{E}_P h_{f_0} + \left( \frac{c_{\sigma} V \tau}{\neff} \right)^{\frac{1}{2 - \vartheta}} 
+ \frac{2 c_{\mathsmaller{B}} \tau}{\neff} + \frac{7 c_{\mathsmaller{B}} B_0 \tau}{4 \neff} + \delta
\nonumber\\
& \phantom{=} + \frac{\Upsilon(f_{D_n, \Upsilon}) + \mathbb{E}_P h_{\wideparen{f}_{D_n, \Upsilon}}}{2} + \frac{5}{2} \varepsilon + \frac{1}{2} r
\nonumber
\end{align*}
holds with probability $\mu$ not less than $1 - 8 C e^{-\tau}$. From the assumptions in \eqref{minradius} follows that
\begin{align*}
\Upsilon(f_{D_n, \Upsilon}) + \mathbb{E}_P h_{\wideparen{f}_{D_n, \Upsilon}}
& \leq \Upsilon(f_0) + 2 \mathbb{E}_P h_{f_0} + 2 r + \delta +
\frac{\Upsilon(f_{D_n, \Upsilon}) + \mathbb{E}_P h_{\wideparen{f}_{D_n, \Upsilon}}}{2} + \frac{5 \varepsilon}{2} 
\end{align*}
holds with probability $\mu$ not less than $1 - 8 C e^{-\tau}$. Consequently, we have
\begin{align*}
\Upsilon(f_{D_n, \Upsilon}) + \mathbb{E}_P h_{\wideparen{f}_{D_n, \Upsilon}} 
\leq 2 \Upsilon(f_0) + 4 \mathbb{E}_P h_{f_0} + 4 r + 5 \varepsilon + 2 \delta.
\nonumber
\end{align*}
Therefore, we have proved the assertion.  
\end{proof} 
 
\bibliographystyle{plain}
\bibliography{FENGBib}
\end{document}